\documentclass[11pt,letterpaper]{mystyle}

\usepackage[utf8]{inputenc}
\usepackage[T1]{fontenc}
\usepackage{xcolor}
\definecolor{cotgray}{RGB}{244,246,248}
\definecolor{cotline}{RGB}{90,99,110}
\definecolor{absblue}{RGB}{46,91,186}
\definecolor{absbluebg}{RGB}{235,242,255}
\definecolor{presgreen}{RGB}{22,130,92}
\definecolor{presgreenbg}{RGB}{233,247,241}
\definecolor{warnorange}{RGB}{193,123,27}
\definecolor{warnorangebg}{RGB}{255,246,229}
\newcommand{\cmark}{\ding{51}}
\newcommand{\xmark}{\ding{55}}

\title{CXR-ContraBench: Benchmarking Negated-Option Attraction in Medical VLMs}
\runningtitle{CXR-ContraBench}

\author{Zhengru Fang et al.}

\newcommand{\paperauthors}{%
Zhengru Fang$^{1}$,
Yanan Ma$^{1}$,
Yu Guo$^{1}$,
Senkang Hu$^{1}$,
Yixian Zhang$^{2}$,
Hangcheng Cao$^{1}$,
Wenbo Ding$^{2}$,
Yuguang Fang$^{1}$\\[0.3em]
\normalfont
$^1$ City University of Hong Kong,
$^2$ Tsinghua Shenzhen International Graduate School%
}

\makeatletter
\renewcommand{\maketitle}{%
\vphantom{a}
\vspace{0mm}
\noindent
\begin{minipage}[t]{0.5\textwidth}
    \includegraphics[height=1.5cm]{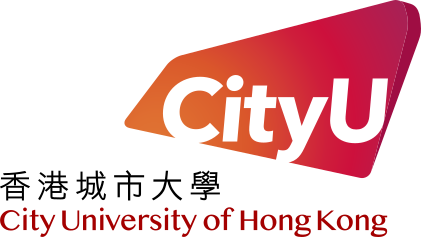}
\end{minipage}%
\hfill
\begin{minipage}[t]{0.5\textwidth}
    \raggedleft
\end{minipage}

\vspace{-5mm}
\noindent\rule{\textwidth}{0.4pt}
\vspace{-10pt}
\vspace{-10pt}

\begin{tcolorbox}[titlebox]
{\raggedright \titlefont \textbf{\@title}\par}
\vspace{4mm}
{\raggedright \paperauthors\par}
\vspace{4mm}
\abscontent
\end{tcolorbox}
\suppressfloats[t]
}
\makeatother

\newcommand{\method}{QCCV-Neg}

\newtcolorbox{promptbox}[1]{%
  enhanced,
  breakable,
  colback=black!2,
  colframe=black!20,
  boxrule=0.5pt,
  arc=2pt,
  left=6pt,
  right=6pt,
  top=5pt,
  bottom=5pt,
  title={#1},
  fonttitle=\bfseries\small,
  coltitle=black,
}

\begin{document}
\begin{abstract}
When a chest X-ray shows consolidation but the question asks which finding is present, a medical vision-language model may answer ``No consolidation''. This is more than an incorrect choice: it is a polarity reversal that emits a clinical statement contradicting the image. We study this failure as \emph{negated-option attraction}, where a model is drawn to a negated answer option even when it conflicts with both the visual evidence and the question. We introduce \textbf{CXR-ContraBench} (Chest X-Ray Contradiction Benchmark), a diagnostic benchmark spanning internal ReXVQA slices and external OpenI and CheXpert protocols. The benchmark centers on present-finding questions, where selecting ``No $X$'' despite visible $X$ creates the main clinical risk, and uses absent-finding questions as secondary tests of whether models copy negated wording. Across CheXpert protocols, the failure is substantial and persistent. On a strict direct presence probe, MedGemma and Qwen2.5-VL reach only 31.49\% and 30.21\% accuracy, respectively; on a matched 135{,}754-record CheXpert training-split protocol, both models select negated options on over 62\% of presence questions. Chain-of-thought prompting reduces some presence-side reversals but does not eliminate them and can amplify absence-side contradictions. Finally, \textbf{\method} (Question-Conditioned Consistency Verifier for Negation) deterministically repairs the measured polarity-confused subset without retraining, raising MedGemma and Qwen2.5-VL to 96.60\% and 95.32\% accuracy on the direct presence probe. These results show that standard accuracy can hide a clinically meaningful inference-time polarity failure. Source code and benchmark construction scripts are available at \url{https://github.com/fangzr/cxr-contrabench-code}.
\end{abstract}

\maketitle
\pagestyle{headstyle}
\thispagestyle{empty}

\section{Introduction}
\label{sec:intro}


Vision-language models (VLMs) have made rapid progress in image understanding and instruction following, as reflected by broad multimodal evaluations~\citep{mmbench,mmmu,mmvet,seedbench}. Medical VLMs (MVLMs) have also improved on radiology interpretation and report generation across domain-specific benchmarks~\citep{pmcvqa,cares,gmaimmbench,benchx,rexvqa,multimedeval,medtrinity,rexgradient}. Yet aggregate accuracy can still mask narrow but clinically meaningful failures. In radiology, this risk is especially acute: a single polarity error can invert the meaning of a finding and produce a misleading clinical statement.

\begin{figure}[t]
  \centering
  \includegraphics[width=\linewidth]{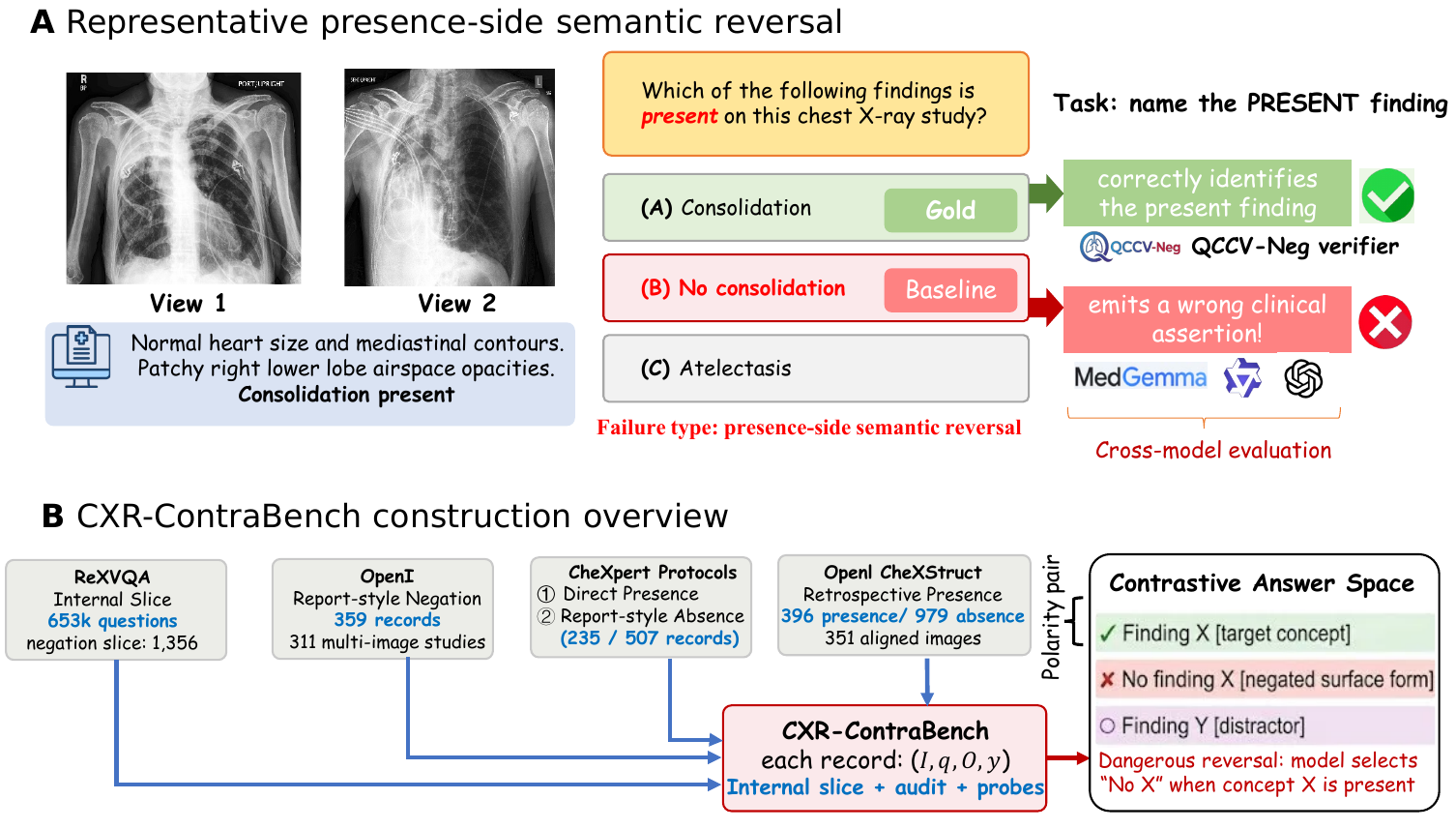}
  \caption{\textbf{Overview of CXR-ContraBench.}
  \textbf{(a)} A representative presence-side semantic reversal. The task asks for the present finding, but a polarity-confused baseline selects the negated option and emits the opposite clinical statement.
  \textbf{(b)} Construction of CXR-ContraBench from internal ReXVQA slices, OpenI report-derived records, OpenI CheXStruct retrospective presence/absence audits, and direct OpenI/CheXpert protocols. The benchmark centers on present-finding questions, where selecting ``No $X$'' when $X$ is visible is the main clinical risk, and also includes absent-finding questions to test whether models copy negated wording instead of the target concept.}
  \label{fig:overview}
\end{figure}

Motivated by the above challenge, we study one such failure mode: \emph{negated-option attraction}.
Its most dangerous form is a presence-side semantic reversal. As illustrated in Figure~\ref{fig:overview}a, while the evaluated MVLMs can often distinguish findings such as consolidation, pleural effusion, or pneumothorax when presented with only positive options, their performance drops sharply once the answer set also includes a negated counterpart, such as ``no consolidation,'' ``no pleural effusion,'' or ``no pneumothorax.'' Instead of selecting the present finding, several strong backbones are drawn to the negated surface form, emitting a clinical statement that directly contradicts the visual evidence.
Moreover, we find that chain-of-thought (CoT)\citep{wei2022chain} prompting fails to reliably resolve this failure. While it reduces the frequency of dangerous reversals on the direct CheXpert\citep{chexpert} presence probe, it does not eliminate them.
Critically, the failure is not a validation-set artifact: on a matched 135{,}754-record CheXpert training protocol, both MedGemma\citep{sellergren2025medgemma} and Qwen2.5-VL\citep{bai2025qwen25vltechnicalreport} emit negated-option selections on over 62\% of presence questions, confirming that the phenomenon is systematic rather than incidental.

While recent medical VQA evaluations broaden coverage, scale, or trustworthiness analysis~\citep{vqarad,slake,pmcvqa,cares,gmaimmbench,benchx,rexvqa}, they do not directly expose this behavior or make negated-option attraction a dedicated benchmark target. NegBench~\citep{negbench} is the closest prior work: it evaluates negation understanding across image, video, and medical datasets, including CheXpert, and shows that CLIP-like joint-embedding models\citep{hafner2021clip} can collapse affirmative and negated statements in embedding space. Winoground~\citep{winoground} further exposes compositional brittleness in vision-language models. CXR-ContraBench differs in scope and failure target. Rather than testing broad negation understanding, we isolate an inference-time answer-space polarity failure in generative MVLMs, where a model may select a negated counterpart even when the question asks for a present chest-X-ray finding. This presence-side semantic reversal is the primary clinical risk we target. What remains missing is a CXR-specific benchmark that makes such polarity-confused predictions auditable across multiple findings, external protocols, scale replication, CoT prompting, and deterministic post-hoc repair.

To fill this gap, we propose \textbf{CXR-ContraBench} (Chest X-Ray Contradiction Benchmark), which, as illustrated in Figure \ref{fig:overview}b, is organized around three protocol families: internal ReXVQA\citep{rexvqa} slices (1k/5k/10k/pooled 3$\times$10k); an OpenI CheXStruct retrospective audit with a 396-example presence subset and a 979-example absence subset; and dedicated external OpenI and CheXpert protocols, including a 235-record direct CheXpert presence probe. Across these protocols, each answer set pairs a finding with its negated form and at least one positive distractor, making polarity-confused predictions directly countable. This controlled design follows the diagnostic logic of prior benchmarks for hallucination~\citep{pope} and compositionality~\citep{winoground}, enabling clean attribution rather than confounding the failure with prompt or format variability.

Our results highlight three contributions. First, we identify negated-option attraction as a systematic inference-time reliability failure in medical VLMs, with presence-side semantic reversal as the primary clinical risk. On the OpenI retrospective audit, MedGemma presence accuracy rises from 29.55\% to 79.29\% after deterministically remapping 197 repairable reversals; on the direct CheXpert presence probe, \method{} (Question-Conditioned Consistency Verifier for Negation) repairs all 153 measured reversals for MedGemma, raising accuracy from 31.49\% to 96.60\%. Second, we introduce CXR-ContraBench as a diagnostic benchmark that makes this failure measurable across retrospective audits, direct probes, controlled question/option rewordings, and a 135{,}754-record scale replication. Third, we show that CoT prompting does not reliably resolve the failure, reducing but not removing dangerous reversals on the presence probe while amplifying contradictions on several absence-side protocols, whereas deterministic verification repairs the measurable subset without retraining. 

\section{Related Work}
\label{sec:related}

Medical VQA and chest-X-ray benchmarks such as VQA-RAD~\citep{vqarad}, SLAKE~\citep{slake}, PMC-VQA~\citep{pmcvqa}, CARES~\citep{cares}, GMAI-MMBench~\citep{gmaimmbench}, BenchX~\citep{benchx}, ReXVQA~\citep{rexvqa}, MultiMedEval~\citep{multimedeval}, MedTrinity-25M~\citep{medtrinity}, and CXReasonBench~\citep{cxreasonbench} expand scale, modality coverage, or trustworthiness analysis. Where these benchmarks pursue broad medical VQA coverage, CXR-ContraBench instead 
targets a single auditable failure: negated-option attraction in generative medical 
VLMs that selects a negated option contradicting the image. It also differs from NegBench~\citep{negbench} in both target model class and failure mechanism. NegBench primarily studies broad negation failures in CLIP-like or joint-embedding VLMs, including affirmation/negation collapse, and explores synthetic-data fine-tuning as a remedy. CXR-ContraBench characterizes inference-time negated-option attraction in generative VLMs under multiple-choice prompting and proposes deterministic post-hoc verification that requires no retraining. The two benchmarks are complementary because embedding-space collapse and answer-space polarity confusion can co-exist in the same model while requiring different interventions.

Diagnostic multimodal benchmarks further motivate targeted probes. Winoground~\citep{winoground}, NegBench~\citep{negbench}, MMBench~\citep{mmbench}, MMMU~\citep{mmmu}, MM-Vet~\citep{mmvet}, and SEED-Bench~\citep{seedbench} show that strong average capability can hide brittleness under controlled tests. The same lesson appears in faithfulness and safety evaluation, including POPE~\citep{pope}, TruthfulQA~\citep{truthfulqa}, HaluEval~\citep{halueval}, FaithDial~\citep{faithdial}, HELM~\citep{helm}, BIG-Bench Hard~\citep{bigbenchhard}, MMLU-Pro~\citep{mmlupro}, and DecodingTrust~\citep{decodingtrust}. CXR-ContraBench brings that diagnostic philosophy into a clinically grounded chest-X-ray setting and couples it with both retrospective and direct polarity probes. A key distinction from prior negation benchmarks lies in our treatment of chain-of-thought prompting. Reasoning is often assumed to improve reliability, but we show that CoT is not a reliable remedy and can substantially worsen errors in some model/protocol pairs. This CoT-induced amplification is a failure mode that has not been documented in prior medical VLM evaluation. Related decoding-time adaptation methods improve LLM task adaptation by modifying the generation distribution~\citep{hu2025distribution}; our focus is different, targeting answer-space polarity confusion in medical VLMs through benchmark construction and deterministic verification rather than distributional decoding adaptation.

Recent medical VLMs and report-oriented resources motivate the relevance of this failure mode but do not isolate it. LLaVA-Med~\citep{llavamed}, MedGemma~\citep{medgemma}, CheXagent~\citep{chexagent}, Med-Flamingo~\citep{medflamingo}, BiomedCLIP~\citep{biomedclip}, MedKLIP~\citep{medklip}, RadFM~\citep{radfm}, XrayGPT~\citep{xraygpt}, RaDialog~\citep{radialog}, and R2GenGPT~\citep{r2gengpt} improve biomedical vision-language capability, while ReXrank~\citep{rexrank}, CoCa-CXR~\citep{cocacxr}, CorBenchX~\citep{corbenchx}, M3D~\citep{m3d}, and DeepTumorVQA~\citep{are3dvlmready} broaden evaluation across report generation, error correction, or 3D diagnosis. Our contribution is complementary: we make both presence-side semantic reversal and absence-side contradiction measurable by construction and auditable at the level of individual answer choices.
Table~\ref{tab:benchmark-comparison} summarizes these distinctions.

\section{CXR-ContraBench}
\label{sec:benchmark}

\subsection{Task formalization}
\label{subsec:task}

CXR-ContraBench is designed to diagnose an inference-time failure of generative medical VLMs under multiple-choice prompting; its measurement targets the model's final selected option rather than the geometry of any joint embedding space. Each benchmark example is a tuple $s=(I,q,\mathcal{O},y)$ with study image set $I$, question $q$, answer set $\mathcal{O}$, and gold option $y$. We focus on contrastive answer spaces containing exactly one explicitly negated option $o^{-}$ (e.g., ``No edema'') and at least one positive finding concept; let $\hat y$ denote the model prediction. In the absence-side diagnostic setting, the question asks which finding concept is absent, and the gold answer is the finding name itself. Because the answer is defined as a concept rather than a report sentence, selecting the negated surface form $o^-$, e.g., ``No edema,'' is counted as a contradiction. In retrospective presence audits such as OpenI CheXStruct, we further distinguish all \emph{wrong with ``No $X$''} cases from the narrower subset of \emph{repairable reversals}, where a unique positive counterpart exists and can be deterministically remapped. This narrow operationalization yields an auditable per-prediction count of polarity-confused choices; individual presence-reversal events also appear in prior binary negation tests~\citep{negbench} but were not isolated as a category.

\subsection{Data sources and protocol construction}
\label{subsec:data}

\textbf{Internal ReXVQA.}
ReXVQA is a large-scale chest-X-ray VQA benchmark with 653,834 questions over 160,000 studies and five task families, including negation detection~\citep{rexvqa}. We use it as the in-distribution setting where aggregate accuracy remains high but explicit-negation contradictions are sparse and auditable. We report a 1k probe, a 5k slice, a 10k slice, and three seeded 10k samples pooled to 30k predictions, retaining the original questions and answer sets while measuring both exact match and contradiction counts.

\textbf{OpenI report-derived negation.}
Our first external protocol is built from OpenI reports~\citep{openi}. We extract study-level absent findings and instantiate ``Which of the following findings is absent on this chest X-ray study?'' The final protocol contains 359 aligned records, including 311 multi-image studies. Targets are pneumothorax (181), pleural effusion (114), focal consolidation (45), and pulmonary edema (19). A positive distractor is selected from co-mentioned present findings, so the finding concept and its negated surface form co-occur in the answer set.

\textbf{OpenI CheXStruct presence/absence.}
Our second external source uses CheXStruct annotations over OpenI studies from CXReasonBench~\citep{cxreasonbench}. It contains 1,375 records over 351 aligned images: 979 absence questions and 396 presence questions. The presence subset provides a retrospective audit of semantic reversal, since the image contains the target finding while the answer space still includes its negated option. Targets are mainly structured thoracic findings such as aortic knob enlargement, descending aortic enlargement, descending aortic tortuosity, tracheal deviation, and ascending aortic enlargement.

\textbf{CheXpert direct presence protocol.}
Our strongest external direct-presence protocol uses the CheXpert validation split~\citep{chexpert}. We group views by study, require the target finding to be cleanly present with at least one different positive distractor, and ask ``Which of the following findings is present on this chest X-ray study?'' Among 200 validation studies, 82 are constructible; per-finding instantiation yields 235 records (mean 2.87 per study). Each answer set contains $X$, its negated counterpart, and another positive finding, so selecting the negated option is an explicit semantic reversal.

\textbf{CheXpert matched positive-only control.}
To test whether direct-presence failures require the negated surface form, we build a matched positive-only control from the same CheXpert source. The question is unchanged, but the answer set contains the true present target and two absent distractors expressed only in positive form. Requiring a paired direct-presence item and two clean absent distractors yields 210 records from 77 studies, enabling paired comparison without changing the image distribution or task framing.

\textbf{CheXpert report-style negation.}
We also construct a CheXpert report-style absence protocol. We require the target finding to be cleanly absent and at least one different finding to be present, then instantiate the study-level absence template. Among 200 validation studies, 132 are constructible, producing 507 records. Targets are pneumothorax (126), consolidation (100), edema (90), pleural effusion (68), cardiomegaly (66), and atelectasis (57). The fixed option layout makes absence-side correction deterministic when the model selects the negated surface form.

\textbf{Why controlled answer spaces matter.}
CXR-ContraBench is a diagnostic instrument rather than a free-form clinical QA corpus. We therefore use controlled answer spaces with explicit polarity and limited lexical variation, following the diagnostic logic of hallucination and compositionality benchmarks~\citep{winoground,pope}. Because the negated option $o^{-}$ and target concept $o^{+}$ co-appear, polarity-confused predictions become directly countable, and both severity and mitigation can be attributed to polarity confusion rather than prompt variability. This design also lets us test CoT effects and non-canonical negation wording; Appendix~\ref{app:paraphrase} shows that paraphrase-aware deterministic matching restores coverage without extra model calls.

\begin{table}[t]
\caption{Qualitative comparison with related benchmarks. CXR-ContraBench uniquely focuses on multi-finding chest-X-ray polarity failure in generative medical VLMs, isolates presence-side reversal, analyzes CoT effects, and supports training-free deterministic repair.}
\label{tab:benchmark-comparison}
\centering
\small
\setlength{\tabcolsep}{3.6pt}
\renewcommand{\arraystretch}{1.10}
\begin{tabular}{@{}lcccccc@{}}
\toprule
\multirow{2}{*}{Benchmark}
  & CXR      & Gen. MVLM & Isolated       & Multi-finding & CoT       & Train-free \\
  & focus    & focus     & presence rev.  & CXR coverage  & analysis  & repair \\
\midrule
VQA-RAD~\citep{vqarad}    & \xmark      & \xmark & \xmark & \xmark & \xmark & \xmark \\
SLAKE~\citep{slake}       & \xmark      & \xmark & \xmark & \xmark & \xmark & \xmark \\
CARES~\citep{cares}       & \xmark      & \cmark & \xmark & \xmark & \xmark & \xmark \\
BenchX~\citep{benchx}     & \cmark      & \xmark & \xmark & \cmark & \xmark & \xmark \\
ReXVQA~\citep{rexvqa}     & \cmark      & \cmark & \xmark & \cmark & \xmark & \xmark \\
NegBench~\citep{negbench} & $\triangle$ & \xmark & \xmark & \xmark & \xmark & \xmark \\
\midrule
\textbf{CXR-ContraBench}
  & \textbf{\cmark} & \textbf{\cmark} & \textbf{\cmark}
  & \textbf{\cmark} & \textbf{\cmark} & \textbf{\cmark} \\
\bottomrule
\end{tabular}\\[3pt]
{\footnotesize $\triangle$ NegBench includes CheXpert negation tasks, but is not CXR-focused and does not isolate presence-side reversal as a clinical audit target.}
\vspace{-4mm}
\end{table}

\section{QCCV-Neg: A Deterministic Verifier}
\label{sec:qccv}

\textbf{Variant notation.}
We use the following notation throughout the experiments.
\ding{172} \textbf{B0} denotes the base multiple-choice prompting output without chain-of-thought or verification.
\ding{173} \textbf{CoT-B0} denotes the same backbone with a fixed hand-written chain-of-thought prompt, again without verification.
\ding{174} \textbf{M1} denotes \method{} applied post hoc to a B0 output.
\ding{175} \textbf{CoT+M1} denotes the same deterministic verifier applied to a CoT-B0 output.
Appendix-only variants are \ding{176} \textbf{Y0}, a confidence-based fallback using structured packet information, and \ding{177} \textbf{M2}, a location-sensitive scaffold for left/right and upper/lower conflicts. Thus B0 and CoT-B0 are model outputs, whereas M1 and CoT+M1 are deterministic post-hoc corrections of those outputs.

\textbf{Deterministic verifier.}
We use \textbf{\method} (Question-Conditioned Consistency Verifier for Negation) as a post-hoc deterministic verifier. Its role is analytic rather than architectural: it tests whether the polarity-confused subset exposed by CXR-ContraBench is precise enough to admit transparent repair. Given a baseline prediction, it activates only when four conditions hold: \ding{172} the answer set contains exactly one explicitly negated option; \ding{173} the baseline predicts that option; \ding{174} question polarity is safely identifiable as absence or presence; and \ding{175} a deterministic remapping exists without concept ambiguity. On absence-side protocols, the override maps to the designated positive target slot; on presence-side protocols, it searches the remaining options for the unique positive counterpart of the same concept. No additional model call is used.

\textbf{Design rationale.}
The design prioritizes auditability, determinism, and causal precision. If the trigger pattern is not exact, the prediction passes through unchanged. Because the verifier operates on controlled answer spaces and post-hoc outputs, each intervention can be attributed to a specific polarity conflict rather than generic reasoning variation. This matters because chain-of-thought can amplify, rather than suppress, the failure (Figure~\ref{fig:qccv}).

\textbf{Auxiliary ablations.}
Y0 and M2 are MedGemma-only ablations reported in the appendix rather than the main figures. Y0 tests whether a generic confidence-based fallback can explain the gains: it replaces B0 only when baseline confidence is below 0.80 and packet confidence exceeds it by at least 0.03. M2 tests whether adding a location-sensitive scaffold helps with left/right and upper/lower conflicts. Neither changes the paper's main conclusion, so the main text focuses on B0, CoT-B0, M1, and CoT+M1.

\begin{figure}[t]
  \centering
  \includegraphics[width=\linewidth]{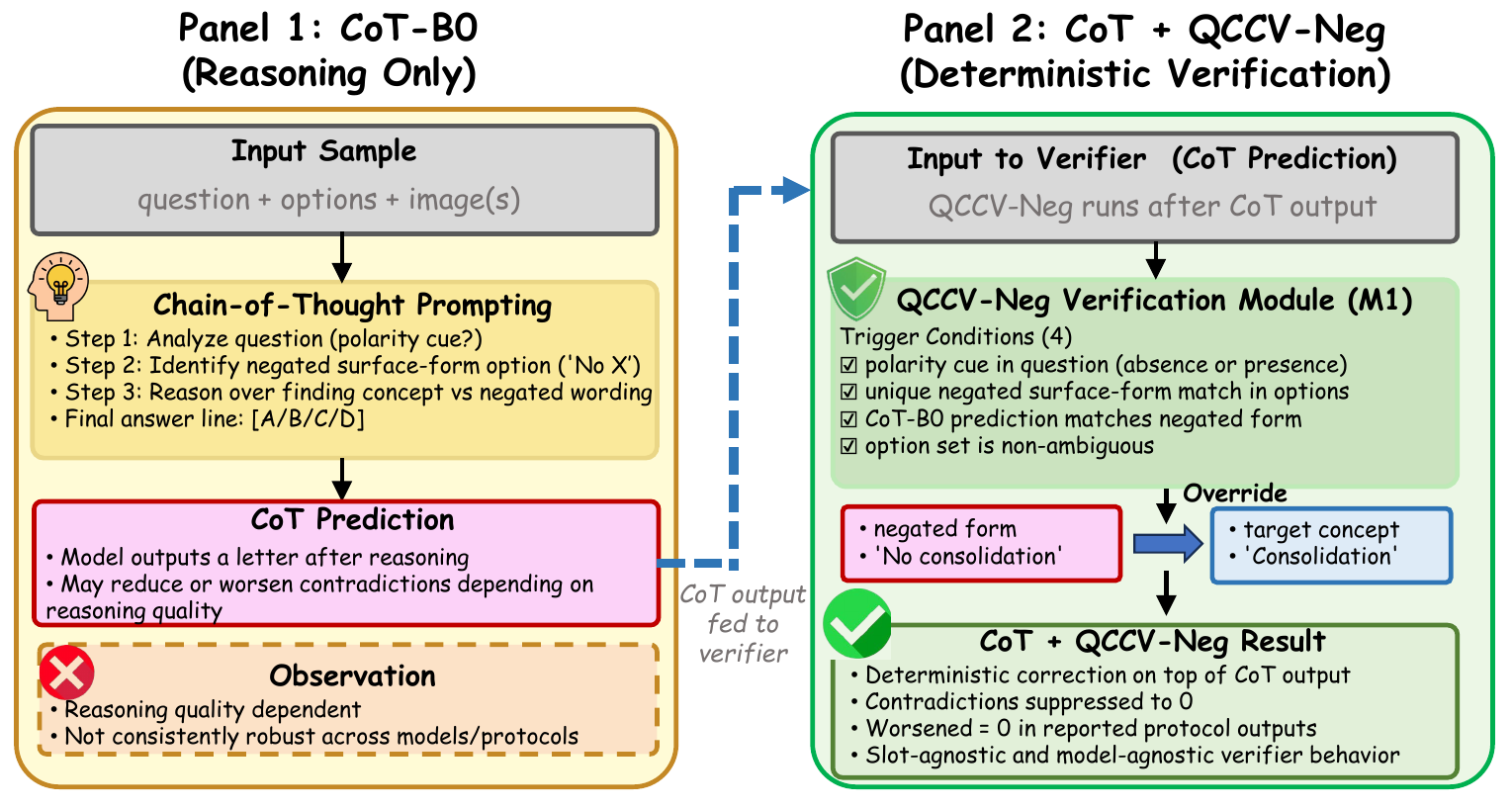}
    \caption{\textbf{Chain-of-thought prompting versus deterministic verification.}
    CoT changes the error distribution but is not a uniformly reliable fix for polarity failures across models and protocols.
    \method{} applies a four-condition trigger after the model output and removes the measured polarity-confused subset deterministically.}
  \label{fig:qccv}
  \vspace{-3mm}
\end{figure}

\section{Experiments}
\label{sec:exp}

\subsection{Setup}

We evaluate MedGemma-4B-it~\citep{medgemma}, LLaVA-1.5-7B, LLaVA-Med-1.5-7B~\citep{llavamed}, CheXagent-2-3B~\citep{chexagent}, and Qwen2.5-VL-7B-Instruct~\citep{qwen2vl} under full-context multiple-choice prompting (B0), with MedGemma-27B-it and GPT-4o added on the direct CheXpert presence probe, matched positive-only control, and external absence protocols. For MedGemma, LLaVA-Med, and Qwen2.5-VL, we also test fixed handwritten CoT prompts (CoT-B0; Appendix Figure~\ref{fig:app-cot-prompt-templates}). Three-shot ICL and in-format LoRA references are reported in Appendix~\ref{app:fewshot} and Appendix~\ref{app:lora}. \method{} (M1) is applied deterministically to source outputs; Y0 and M2 appear only in the appendix.

All results are exact-match multiple-choice evaluations. We report accuracy, prediction changes, intervention coverage, and polarity-error counts: contradictions for absence-side protocols and negated-option reversals for presence-side protocols. Retrospective audits additionally distinguish all wrong ``No $X$'' cases from repairable reversals. For absence-side external and internal B0/M1 comparisons, we report 95\% bootstrap CIs with paired permutation tests; CheXpert absence also uses study-clustered CIs. Direct presence probes are reported as point estimates.

\subsection{Main external results}

\begin{figure*}[t]
  \centering
  \includegraphics[width=\textwidth]{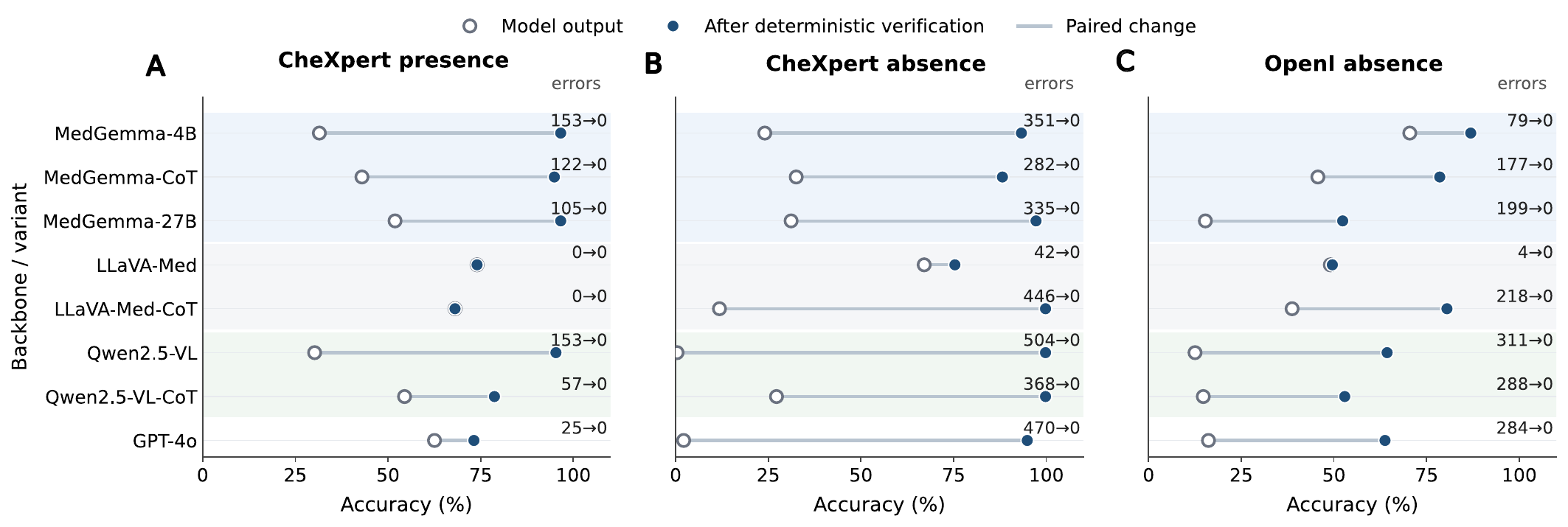}
  \caption{\textbf{Main external results.}
Open circles denote model outputs (B0 or CoT-B0), filled circles denote verified outputs (M1 or CoT+M1), and right-side labels show polarity errors before$\rightarrow$after.
Errors are dangerous ``No $X$'' reversals for direct presence and contradictions for absence protocols.
Full results are in Appendix Table~\ref{tab:external-results-full}.}
  \label{fig:external-results}
  \vspace{-2mm}
\end{figure*}

The main clinical-risk evidence comes from the direct CheXpert presence probe in Figure~\ref{fig:external-results}a. The failure is strong but not universal. MedGemma-4B and Qwen2.5-VL start at only 31.49\% and 30.21\% B0 accuracy and emit 153 dangerous negated-option reversals each; MedGemma-27B still emits 105 reversals at 51.91\% B0, while GPT-4o emits 25 at 62.55\%. \method{} repairs all measured reversals for these four backbones, raising accuracy to 96.60\%, 95.32\%, 96.60\%, and 73.19\%, respectively. By contrast, LLaVA-Med shows no measured direct presence reversal on this construction, so M1 leaves both its B0 and CoT-B0 accuracy unchanged. Other backbones, including LLaVA-1.5 and CheXagent, show smaller but still measurable reversal burdens; the full matrix is reported in Appendix Table~\ref{tab:external-results-full}. The direct presence result is therefore best understood as a backbone-dependent reliability failure rather than a universal property of every medical VLM.

Moreover, the retrospective OpenI presence audit shows that the same semantic reversal extends beyond a single direct protocol. MedGemma selects the negated option on 270 of 279 wrong presence questions (96.77\%). Under the deterministic matching rule, 197 of these are repairable reversals, raising accuracy from 29.55\% to 79.29\%; 73 wrong ``No $X$'' cases remain unresolved after M1. The wrong-answer negated-option rate remains similarly high for Qwen2.5-VL at 92.58\%, but it is lower for LLaVA-1.5 and LLaVA-Med at 41.61\% and 21.19\%, respectively. This audit reinforces the main point from the direct probe: presence-side reversal is systematic for some backbones, clinically risky when it occurs, and not well described by a single average across models.

A matched positive-only control sharpens causal interpretation (Appendix Table~\ref{tab:presence_positive_control}). On the paired 210-record subset, removing the negated option raises accuracy from 30.00\% to 56.67\% for MedGemma-4B, from 30.00\% to 46.19\% for Qwen2.5-VL, and from 50.95\% to 60.95\% for MedGemma-27B, but lowers it for GPT-4o (62.38\%$\rightarrow$45.71\%). This shows that the negated option is a major error driver for some backbones, while others retain broader present-finding discrimination errors even without negated wording.

Chain-of-thought prompting still does not substitute for polarity control. On the direct presence probe, MedGemma-CoT improves over MedGemma B0 in raw accuracy but still leaves 122 dangerous reversals, and Qwen2.5-VL-CoT still leaves 57. On this construction, CoT reduces the burden for these two backbones but does not resolve it. Figure~\ref{fig:external-results}b further shows that CoT can amplify negated-option attraction sharply on absence-side protocols, with LLaVA-Med rising from 42 to 446 contradictions on CheXpert. There, explicit negation contradiction remains cross-model but highly variable in magnitude: on CheXpert report-style negation, baseline accuracy ranges from 0.39\% to 67.06\% with 42 to 504 contradictions; on OpenI report-derived negation, it ranges from 2.51\% to 70.47\% with 4 to 311 contradictions. Across the completed absence-side external variants, \method{} eliminates every measured contradiction and yields positive accuracy deltas everywhere.

Full three-shot ICL results appear in Appendix~\ref{app:fewshot}; deterministic verification remains the only intervention that consistently removes the measured polarity-confused subset.

\subsection{Scale replication confirms systematic failure}

To verify that the observed failures are not artifacts of small validation samples, we construct a matched CheXpert training-split direct presence protocol with 135{,}754 records from 59{,}173 studies. MedGemma reaches 33.05\% B0 accuracy with 84{,}966 dangerous negated-option reversals, and Qwen2.5-VL reaches 30.47\% with 88{,}668 reversals; \method{} removes all of them and raises accuracy to 95.63\% and 95.78\%. The absence-side training-split replication remains consistent with the same mechanism: on 25{,}900 records, MedGemma rises from 49.67\% to 86.37\% and Qwen2.5-VL rises from 14.73\% to 64.54\%, again with all measured contradictions removed. Appendix Table~\ref{tab:chexpert-train-scale} and Figure~\ref{fig:app-scale-replication} give the scale-replication breakdown. We additionally evaluate MedGemma-27B-it and GPT-4o on the direct presence probe, the matched positive-only control, and the same external absence protocols; extended analysis and the matched positive-only control results appear in Appendix~\ref{app:large_models}.

\subsection{Supporting internal results}

Internal ReXVQA results (Appendix Table~\ref{tab:internal-results}) confirm that \method{} remains sparse and targeted in-distribution. On the five-backbone 1k probe, intervention coverage stays below 1\% and worsened remains 0 throughout; longer MedGemma-only runs at 5k, 10k, and pooled 30k retain the same picture. Contradictions drop to 0 in every run, but aggregate accuracy gains are small and non-significant under permutation testing, consistent with the sparse-intervention design.

The MedGemma-only ablations in Appendix~\ref{app:stability} show why generic structured packet fallback is not enough. Y0 changes 74 predictions on the 5k slice yet slightly decreases overall accuracy and leaves contradiction counts unchanged, while M2 adds no gain over M1 on the 1k probe.

\section{Analysis}
\label{sec:analysis}

\subsection{Aggregate accuracy hides polarity failure; retrospective audits expose its clinical form}
\begin{figure*}[htbp]
  \centering
  \includegraphics[width=\linewidth]{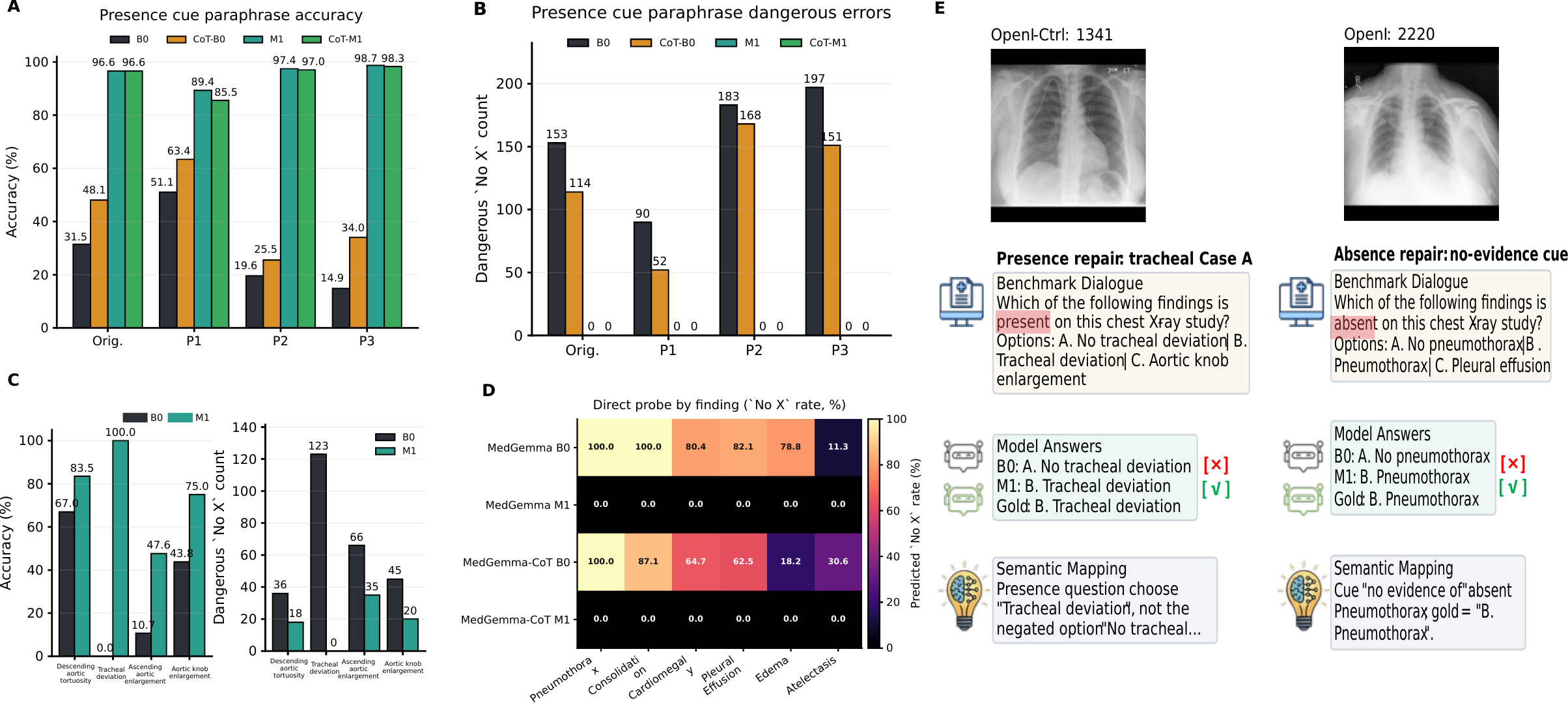}
  \caption{Main robustness atlas. \textbf{A--B}: Presence-cue paraphrase accuracy and dangerous negated-option counts across the original prompt and three paraphrases. \textbf{C}: OpenI CheXStruct finding-level accuracy and dangerous reversal counts under B0 and M1. \textbf{D}: Direct CheXpert presence-probe negated-option rate by finding and backbone. \textbf{E}: Representative presence-repair and absence-repair cases.}
  \label{fig:contradictions}
\end{figure*}
ReXVQA already contains the central discrepancy that motivates this paper. On the 5k slice, B0 reaches 82.42\% overall but only 43.13\% on explicit negation prompts (pooled 30k: 80.36\% vs.\ 42.18\%). The gap is clinically sharper on direct presence probes. On CheXpert direct presence, MedGemma achieves only 31.49\% and Qwen2.5-VL achieves only 30.21\%, even though both models remain competitive on broader chest-X-ray benchmarks. Contradiction-sensitive evaluation therefore expands a region of model behavior that aggregate metrics compress into a small residual error term.
Moreover, the OpenI CheXStruct split now serves two purposes. Its absence subset confirms the absence-side surface-form contradiction pattern (MedGemma: 19.00\%$\rightarrow$39.94\%). More importantly, its 396-example presence subset directly exposes the clinically risky presence-side semantic reversal. MedGemma selects the negated option on 270 of 279 wrong presence questions, consistent with the 96.77\% rate on this protocol. Of these, 197 satisfy the exact positive-counterpart criterion and are therefore repairable reversals under the deterministic rule, raising accuracy from 29.55\% to 79.29\%; 73 wrong ``No $X$'' cases remain after M1. The same wrong-answer pattern also appears across other backbones: Qwen2.5-VL emits negated-option predictions on 92.58\% of wrong presence cases, LLaVA-1.5 on 41.61\%, and LLaVA-Med on 21.19\% (Table~\ref{tab:openi-control}). Because the presence and absence subsets share source distribution, format, and scoring, this audit localizes the effect to polarity handling rather than to generic external noise.

\subsection{Chain-of-thought does not resolve polarity failure: cross-model evidence and error taxonomy}

The cross-model results show that chain-of-thought prompting does not reliably improve reliability in the negated-option-attraction regime. On the direct CheXpert presence probe, MedGemma-CoT still leaves 122 dangerous reversals and Qwen2.5-VL-CoT still leaves 57 despite improving raw accuracy over B0. On the absence-side protocols, the effect ranges from mixed to severe: MedGemma worsens on OpenI, while LLaVA-Med rises from 42 to 446 contradictions on CheXpert. Manual inspection suggests three recurring patterns: presence-side semantic reversal, absence-side selection of the negated surface form, and a smaller group where removing the negated option still leaves an incorrect distractor choice.

This profile is also not uniform across findings or prompts. Figure~\ref{fig:contradictions} collects four supporting views beyond the headline direct-probe result. Panels~A--B show that once the trigger vocabulary is extended beyond \emph{absent}, dangerous presence-side reversals are suppressed to 0 across the original prompt and all three paraphrase variants, with accuracy improving in parallel. Panel~C shows that the OpenI CheXStruct presence failure is concentrated in specific findings rather than spread uniformly across categories. Panel~D further shows that direct-presence negated-option rates are strongly heterogeneous across both findings and backbones. Panel~E grounds the benchmark logic in representative presence-repair and absence-repair cases, making the deterministic remapping path explicit. Appendix Table~\ref{tab:presence-finding-breakdown}, Table~\ref{tab:finding-breakdown}, and Figure~\ref{fig:app-finding-error-atlas} provide the full per-finding breakdown.

\subsection{Verifier sparsity, robustness, and generalization scope}
The internal experiments show that sparsity is deliberate rather than incidental. Because the verifier activates only on an exact answer-space signature, most predictions pass through unchanged. The 1k probe has no worsened case, and the pooled 30k run has exactly one, arising from a differential-diagnosis prompt phrased as ``least likely present'' (Appendix Figure~\ref{fig:app-internal-sparsity-atlas}).

Paraphrase robustness improves once the trigger vocabulary is extended beyond the literal negation cue. Table~\ref{tab:main-paraphrase-robustness} summarizes the CheXpert report-style absence stress test. B0 retains a large contradiction burden under all rewrites, and the original literal-surface trigger misses P1 and P3. By contrast, the extended trigger removes all measured contradictions across the original wording and three paraphrases, with no worsened case. Full direct-presence paraphrase results and layout-robustness checks remain in Appendix~\ref{app:paraphrase} and Appendix Figure~\ref{fig:app-robustness-atlas}.

\begin{table}[t]
\caption{Main paraphrase-robustness summary on CheXpert report-style absence. The original trigger misses P1 and P3, while the extended trigger removes all measured contradictions with no worsened case.}
\label{tab:main-paraphrase-robustness}
\centering
\small
\setlength{\tabcolsep}{3.6pt}
\renewcommand{\arraystretch}{1.08}
\begin{tabular}{@{}lccccc@{}}
\toprule
Variant & B0 Acc. & B0 Contrad. & Orig. M1 Contrad. & Ext. M1 Acc. & Ext. M1 Contrad. \\
\midrule
Original & 24.06 & 351 & 0   & 93.29 & 0 \\
P1       & 22.68 & 283 & 283 & 78.50 & 0 \\
P2       & 18.93 & 381 & 0   & 94.08 & 0 \\
P3       & 24.65 & 379 & 379 & 99.41 & 0 \\
\bottomrule
\end{tabular}
\vspace{-1mm}
\end{table}


\section{Conclusion}
\label{sec:conclusion}

This paper establishes negated-option attraction as a distinct, inference-time reliability failure in medical VLMs. Its most clinically dangerous form is presence-side semantic reversal, where a model selects the negated option even though the image contains finding $X$ and the question asks which finding is present. This failure can be hidden by aggregate accuracy and is not reliably resolved by chain-of-thought prompting. \method{} shows that negated-option attraction is measurable by construction and correctable at inference time without retraining. Because the verifier operates post hoc on the model's final prediction, it intercepts the failure regardless of whether the prediction was produced by direct prompting or by chain-of-thought reasoning. This property matters precisely because CoT does not reliably resolve the failure and can amplify it on the absence-side diagnostics. The verifier is not intended to improve general radiology reasoning; it targets a specific answer-space polarity conflict. That target-specific design enables auditability, deterministic behavior, and very low intervention coverage on internal ReXVQA runs.

\textbf{Broader impact and future scope.}
CXR-ContraBench makes a clinically meaningful silent error measurable: a model may choose a negated answer option, including non-canonical paraphrases, when the image and question support the opposite polarity. Its positive impact is to expose polarity failures that aggregate accuracy can hide. Beyond this benchmark, future work should examine the cross-modal representation geometry of polarity, including how visual findings and negated text are aligned in high-dimensional semantic spaces.



\bibliographystyle{refstyle}
\bibliography{reference}
\newpage

\appendix

%

\clearpage
\newpage
\appendix

\begin{center}
\Large\textbf{CXR-ContraBench: Benchmarking Negated-Option Attraction in Medical VLMs}

\vspace{0.3em}

\Large Supplementary Material
\end{center}

\vspace{0.8em}

\noindent\textbf{\Large Contents}

\vspace{0.3em}

\begin{itemize}
  \item \textbf{A~~Score-Level Analysis of Negated-Option Attraction} \dotfill \textbf{\pageref{app:theory_local}}

  \item \textbf{B~~Reproducibility and Compute} \dotfill \textbf{\pageref{app:repro}}

  \item \textbf{C~~Statistical Uncertainty of Main Results} \dotfill \textbf{\pageref{app:stats}}

  \item \textbf{D~~Additional Internal Stability Results} \dotfill \textbf{\pageref{app:stability}}

  \item \textbf{E~~Additional External Analyses} \dotfill \textbf{\pageref{app:external}}

  \item \textbf{F~~Large-Backbone Evaluation: MedGemma-27B and GPT-4o} \dotfill \textbf{\pageref{app:large_models}}

  \item \textbf{G~~Few-Shot In-Context Learning Baseline} \dotfill \textbf{\pageref{app:fewshot}}

  \item \textbf{H~~LoRA Fine-tuning Baseline} \dotfill \textbf{\pageref{app:lora}}

  \item \textbf{I~~Negation Paraphrase Stress Test} \dotfill \textbf{\pageref{app:paraphrase}}

  \item \textbf{J~~Responsible Release, Licenses, and Broader Impact} \dotfill \textbf{\pageref{app:licenses}}
\end{itemize}

\vspace{0.6em}

\section{Score-Level Analysis of Negated-Option Attraction}
\label{app:theory_local}

This appendix gives a local score-level account of negated-option attraction in the controlled multiple-choice setting of CXR-ContraBench. The claim is deliberately narrow. We do not propose an identified theory of VLM internals or a general account of negation understanding outside the benchmark's contrastive answer spaces. The goal is only to formalize a small and auditable failure pattern exposed directly by the benchmark: the model can keep the correct finding concept competitive in the answer set while still assigning the highest final score to its explicitly negated surface form. On a carefully delimited subset, that pattern also admits a deterministic remapping rule.

\subsection{Setup and Notation}

Let $\mathcal S$ denote the benchmark items considered in this appendix, let $\mathcal I$ and $\mathcal Q$ denote the image and question domains, let $\Omega$ denote a global option universe, and let $\mathcal C$ denote the global finding-concept space. Each item $s\in\mathcal S$ is written as $s=(I_s,q_s,\mathcal O_s,y_s)$, where $I_s\in\mathcal I$ is the study image set, $q_s\in\mathcal Q$ is the multiple-choice question, $\mathcal O_s\subseteq\Omega$ is the finite answer set, and $y_s\in\mathcal O_s$ is the gold option. Let $f_\theta:\mathcal I\times\mathcal Q\times\Omega\to\mathbb R$ be the model score function under parameters $\theta$, and let the induced prediction be $\hat y_\theta(s):=\arg\max_{o\in\mathcal O_s} f_\theta(I_s,q_s,o)$, with a fixed deterministic tie-breaking rule whenever the maximum is not unique. This notation keeps the item-level answer set $\mathcal O_s$ separate from the global option space $\Omega$.

Each option $o\in\Omega$ carries two attributes. The map $c(o)\in\mathcal C$ gives the finding concept, and the map $\rho(o)\in\{\mathrm{pos},\mathrm{neg}\}$ gives the surface polarity. Thus $\rho(o)=\mathrm{pos}$ denotes a positive wording such as ``edema,'' whereas $\rho(o)=\mathrm{neg}$ denotes an explicitly negated wording such as ``No edema.'' By construction, every answer set $\mathcal O_s$ contains exactly one explicitly negated option; denote it by $o_s^-$. If there exists a unique option $o_s^+\in\mathcal O_s\setminus\{o_s^-\}$ such that $c(o_s^+)=c(o_s^-)$ and $\rho(o_s^+)=\mathrm{pos}$, call $o_s^+$ the positive counterpart of $o_s^-$. Define $\mathcal S_{\mathrm{pair}}:=\{s\in\mathcal S:o_s^+\text{ exists uniquely}\}$ and $\mathcal S_{\mathrm{gold+}}:=\{s\in\mathcal S_{\mathrm{pair}}:y_s=o_s^+\}$. For the direct-presence and absence-side protocols analyzed here, the repairable cases lie in $\mathcal S_{\mathrm{gold+}}$: the negated option has a unique positive counterpart and that counterpart is the gold answer.

To analyze why such errors arise, we use the toy score decomposition $f_\theta(I,q,o)=u_\theta(I,q,c(o))+v_\theta(I,q,c(o),\rho(o))$, where $u_\theta(I,q,c)$ is a concept-support term capturing how strongly the image-question pair supports concept $c$, and $v_\theta(I,q,c,\rho)$ is a polarity-specific residual term capturing the remaining preference for one surface realization once the concept is fixed. This factorization is introduced only as an explanatory lens. It should be read at the level of option scores or conditional token probabilities, not as a literal claim about hidden-state geometry or the exact internal computation performed by the VLM. Throughout this appendix, each analyzed answer set contains at least one distractor option besides the concept-matched pair $\{o_s^+,o_s^-\}$.

\subsection{Local Score-Level Account}

\begin{proposition}[Concept preservation with within-concept polarity bias]
\label{prop:concept_polarity_dense}
Let $s\in\mathcal S_{\mathrm{gold+}}$, and write $x_s:=c(y_s)$. Suppose that
\[
\max\bigl\{f_\theta(I_s,q_s,o_s^+),\,f_\theta(I_s,q_s,o_s^-)\bigr\}>
\max_{o\in\mathcal O_s\setminus\{o_s^+,o_s^-\}} f_\theta(I_s,q_s,o)
\]
and that
\[
v_\theta(I_s,q_s,x_s,\mathrm{neg})>v_\theta(I_s,q_s,x_s,\mathrm{pos}).
\]
Then $\hat y_\theta(s)=o_s^-$.
\end{proposition}

\begin{proof}
Because $s\in\mathcal S_{\mathrm{gold+}}$, we have $y_s=o_s^+$ and $c(o_s^+)=c(o_s^-)=x_s$. The decomposition therefore gives $f_\theta(I_s,q_s,o_s^+)=u_\theta(I_s,q_s,x_s)+v_\theta(I_s,q_s,x_s,\mathrm{pos})$ and $f_\theta(I_s,q_s,o_s^-)=u_\theta(I_s,q_s,x_s)+v_\theta(I_s,q_s,x_s,\mathrm{neg})$. The first assumption implies that no distractor option in $\mathcal O_s\setminus\{o_s^+,o_s^-\}$ can attain the global maximum, so the prediction must lie in $\{o_s^+,o_s^-\}$. The second assumption yields $f_\theta(I_s,q_s,o_s^-)>f_\theta(I_s,q_s,o_s^+)$, which forces $\hat y_\theta(s)=o_s^-$.
\end{proof}

Proposition~\ref{prop:concept_polarity_dense} isolates the local mechanism that matters for this benchmark. The correct finding concept need not lose to a distractor concept for the final answer to be wrong. It is enough that the concept-matched pair $\{o_s^+,o_s^-\}$ collectively outrank the distractors while the within-concept polarity term favors the negated wording. In that regime, the answer space remains anchored on the right finding concept, but the final decision shifts from the clinically correct positive realization to the explicitly negated one. The error is therefore concept-preserving at the level of the contrastive answer space even though it is answer-level incorrect. This interpretation is further supported by the matched positive-only control in Table~\ref{tab:presence_positive_control}: removing the negated option raises accuracy from 30.00\% to 56.67\% for MedGemma-4B and from 50.95\% to 60.95\% for MedGemma-27B on the paired subset, indicating that the negated surface form itself is a major driver of error for these backbones.

A minimal three-option example makes the mechanism concrete. Let $\mathcal O_s=\{o_s^+,o_s^-,o_{\mathrm{dist}}\}$ with $o_s^+=(x,\mathrm{pos})$, $o_s^-=(x,\mathrm{neg})$, and $o_{\mathrm{dist}}=(z,\mathrm{pos})$ for some $z\neq x$. If $u_\theta(I_s,q_s,x)=5.0$, $v_\theta(I_s,q_s,x,\mathrm{pos})=0.3$, $v_\theta(I_s,q_s,x,\mathrm{neg})=1.2$, $u_\theta(I_s,q_s,z)=3.0$, and $v_\theta(I_s,q_s,z,\mathrm{pos})=0.8$, then the three option scores are $5.3$, $6.2$, and $3.8$, respectively. The concept paired with the gold answer still dominates the distractor concept represented in the answer set, yet the negated surface form receives the largest final score. Figure~\ref{fig:contradictions}D further shows that the direct-presence negated-option rate varies sharply across findings and backbones, consistent with a concept-dependent polarity term $v_\theta(I,q,c,\rho)$ rather than a uniform global bias.

\begin{proposition}[Auditability of negated-option errors]
\label{prop:auditable_dense}
Let $s\in\mathcal S_{\mathrm{gold+}}$, and suppose that $o_s^-$ is the unique option in $\mathcal O_s$ satisfying $c(o)=c(y_s)$ and $\rho(o)=\mathrm{neg}$. Then $\hat y_\theta(s)=o_s^-$ if and only if $c(\hat y_\theta(s))=c(y_s)$ and $\rho(\hat y_\theta(s))\neq\rho(y_s)$.
\end{proposition}

\begin{proof}
Since $s\in\mathcal S_{\mathrm{gold+}}$, we have $y_s=o_s^+$ and hence $\rho(y_s)=\mathrm{pos}$. If $\hat y_\theta(s)=o_s^-$, then $c(\hat y_\theta(s))=c(o_s^-)=c(y_s)$ and $\rho(\hat y_\theta(s))=\mathrm{neg}\neq\mathrm{pos}=\rho(y_s)$. Conversely, suppose that $c(\hat y_\theta(s))=c(y_s)$ and $\rho(\hat y_\theta(s))\neq\rho(y_s)$. Because $\rho(y_s)=\mathrm{pos}$, it follows that $\rho(\hat y_\theta(s))=\mathrm{neg}$. By the stated uniqueness assumption, the only option in $\mathcal O_s$ with concept $c(y_s)$ and negative surface polarity is $o_s^-$, so $\hat y_\theta(s)=o_s^-$.
\end{proof}

Proposition~\ref{prop:auditable_dense} explains why this failure is countable rather than merely descriptive. In an unrestricted answer space, concept confusion and polarity confusion can be entangled in ways that are hard to separate from generic reasoning error. Here the answer set explicitly contains both the gold concept and its negated surface form, so the polarity flip becomes a discrete event. The same formal event has different clinical readings across protocols. When the question asks which finding is present, selecting the concept-matched negated option is a presence-side semantic reversal. When the question asks which finding is absent, the same event becomes an absence-side contradiction that copies the negative surface form instead of naming the absent concept. The auditability claim is also reflected in the retrospective OpenI presence audit in Table~\ref{tab:openi-control}, where MedGemma selects the negated option on 270 of 279 wrong presence questions, indicating that the dominant failure on this protocol is a countable concept-matched polarity flip rather than arbitrary distractor confusion.

To match the implementation without overloading the formalism with parsing details, let $\mathcal S_{\mathrm{safe}}(\theta)\subseteq\mathcal S_{\mathrm{gold+}}$ denote the subset of items on which the deployed verifier trigger fires. This subset incorporates the conservative cue checks used in practice, including the requirement that the baseline prediction is the unique negated option and that a deterministic positive counterpart is available. We define the verifier output by
\[
T_\theta(s):=
\begin{cases}
o_s^+, & s\in\mathcal S_{\mathrm{safe}}(\theta),\\
\hat y_\theta(s), & s\notin\mathcal S_{\mathrm{safe}}(\theta).
\end{cases}
\]
By construction, $T_\theta$ is exact on $\mathcal S_{\mathrm{safe}}(\theta)$ and leaves the baseline prediction unchanged outside that subset.
At the dataset level, the aggregate accuracy change can be written as $\Delta_{\mathrm{acc}}(\theta)=\bigl(|\mathcal S_{\mathrm{help}}(\theta)|-|\mathcal S_{\mathrm{harm}}(\theta)|\bigr)/|\mathcal S|$, where $\mathcal S_{\mathrm{help}}(\theta)$ denotes items corrected by verification and $\mathcal S_{\mathrm{harm}}(\theta)$ denotes items worsened by it. By the definition of $\mathcal S_{\mathrm{safe}}(\theta)$ together with $T_\theta(s)$, we have $\mathcal S_{\mathrm{harm}}(\theta)=\emptyset$ and $\mathcal S_{\mathrm{help}}(\theta)=\mathcal S_{\mathrm{safe}}(\theta)$ on the studied protocols. This is the formal sense in which the measured repairable subset admits a deterministic post-hoc correction without observed regressions on the studied constructions. This behavioral pattern is reflected in the direct CheXpert presence results in Figure~\ref{fig:external-results}a and Table~\ref{tab:external-results-full}: MedGemma-4B and Qwen2.5-VL each emit 153 negated-option reversals at baseline, and deterministic remapping removes all of them, raising accuracy from 31.49\% to 96.60\% and from 30.21\% to 95.32\%, respectively.

\subsection{Theory--Experiment Bridge}
\label{app:theory_bridge}
\begin{figure}[t]
  \centering
  \includegraphics[width=\linewidth]{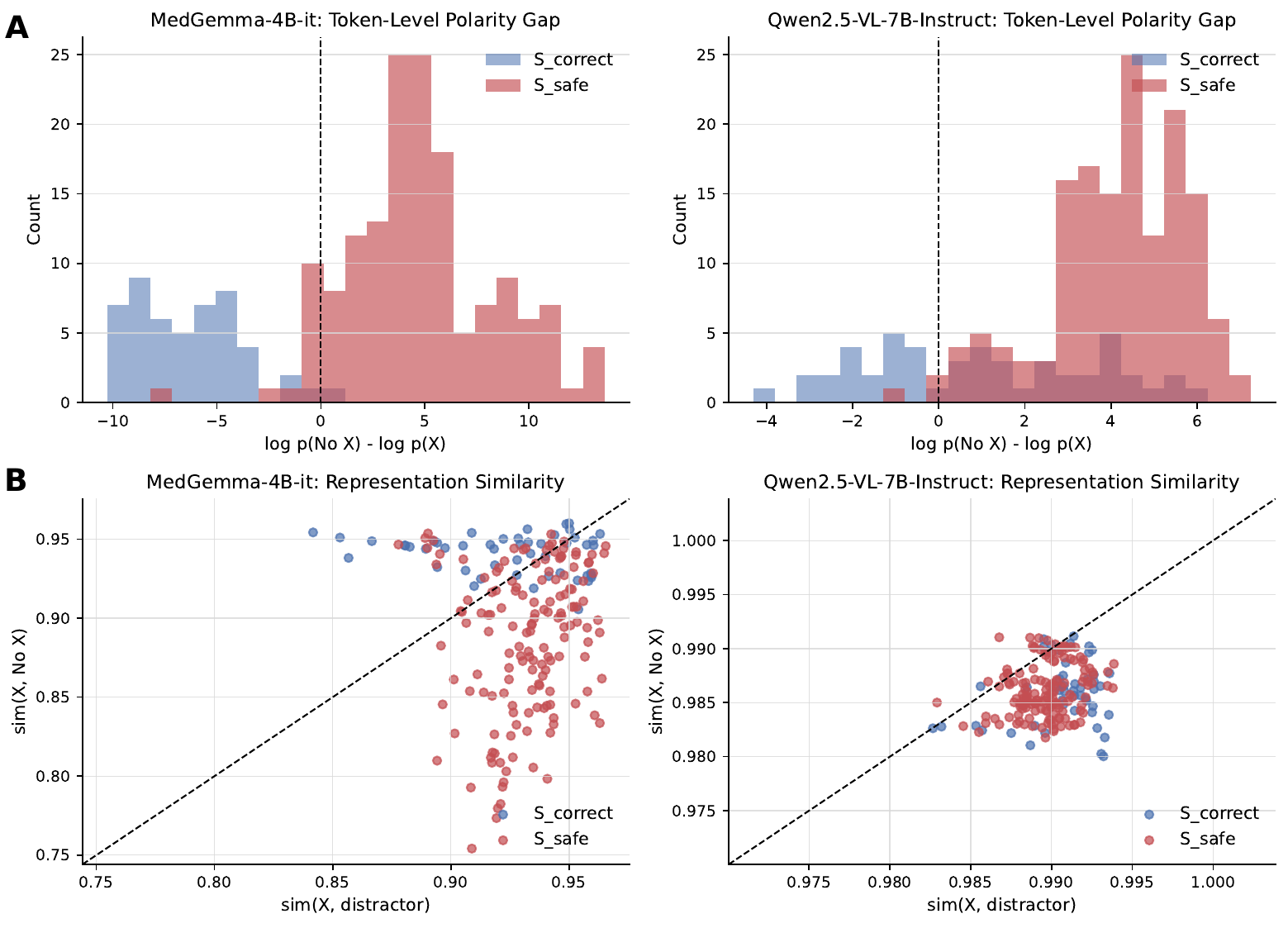}
  \caption{\textbf{Theory--experiment bridge for Proposition~\ref{prop:concept_polarity_dense}.} \textbf{A: token-level polarity gap.} For the direct-presence safe subset $\mathcal S_{\mathrm{safe}}(\theta)$, where the baseline predicts the negated option but deterministic remapping recovers the gold positive counterpart, the distribution of the measured log-probability gap between the negated surface form and its positive counterpart shifts strongly to the positive side for both MedGemma-4B-it and Qwen2.5-VL-7B-Instruct. By contrast, a correctly answered control subset with the same paired-answer structure is concentrated near or below zero. \textbf{B: representation similarity.} A simple last-token hidden-state probe does not recover the same separation: similarity between the positive and negated options is not consistently higher than similarity to the distractor under the current readout. We therefore interpret Proposition~\ref{prop:concept_polarity_dense} as a local score-level account supported by token probabilities, while treating the current representation-level evidence as a boundary condition rather than confirmatory mechanism evidence.}
  \label{fig:theory_bridge}
\end{figure}

The decomposition above makes a concrete empirical prediction. On repairable direct-presence failures, the option-level scorer should favor the negated surface form even though the underlying finding concept remains competitive against the distractor concepts available in the answer set. Figure~\ref{fig:theory_bridge} evaluates that prediction on $\mathcal S_{\mathrm{safe}}(\theta)$, namely the items for which the baseline predicts the negated option and deterministic remapping recovers the unique positive counterpart. The figure also compares this safe subset with a correctly answered control subset having the same paired-answer structure. The goal of this subsection is not to prove the decomposition from data, but to show that the measured token-level behavior matches the local score-level account closely enough to support it as an explanatory bridge.

The top row (Fig. \ref{fig:theory_bridge}A) audits polarity at the token level by plotting the measured log-probability gap between the negated option string and its positive counterpart on $\mathcal S_{\mathrm{safe}}(\theta)$ and on the matched correctly answered control subset. If Proposition~\ref{prop:concept_polarity_dense} is the right local account, the safe subset should display a strong positive shift at the option level even when the correct finding concept remains viable. That is what we observe. On $\mathcal S_{\mathrm{safe}}(\theta)$, the mean option-level gap is 4.790 nats for MedGemma-4B-it and 4.116 nats for Qwen2.5-VL-7B-Instruct, whereas the corresponding concept-only continuation difference between positive and negated contexts is only 0.000 and 0.001 nats.

The bottom row (Fig. \ref{fig:theory_bridge}B) provides a useful boundary condition. A simple last-token hidden-state probe does not recover the same separation: similarity between the positive and negated options is not consistently higher than similarity to the distractor under the current readout. We therefore do not interpret Proposition~\ref{prop:concept_polarity_dense} as a claim about any particular hidden-state geometry. The strongest support for the appendix theory comes from option scores and conditional token probabilities, which align with the score-level account. Under the current probe, the representation result is informative mainly because it marks a limit on what the present evidence can justify.

This analysis is intentionally local. The factorization $f_\theta=u_\theta+v_\theta$ is a toy decomposition introduced for interpretability, not an identified model of the VLM internals. The propositions apply only to controlled contrastive answer spaces in which an explicitly negated option and a unique positive counterpart are both present. They do not address free-form report generation, implicit negation, or broader clinical reasoning. The appendix therefore supports a narrow claim: within this benchmark, concept-level competitiveness and answer-level correctness can diverge in a structured, auditable way, and the repairable subset of that divergence admits a transparent deterministic correction.

\section{Reproducibility and Compute}
\label{app:repro}

This section documents the public benchmark substrate, script-level reproducibility details, and compute configuration underlying all reported experiments.

\subsection{Benchmark Composition}

Figure~\ref{fig:app-benchmark-composition} provides a five-panel atlas of the public data substrate behind the main external analyses.
Panel~A shows that the OpenI CheXStruct retrospective audit split is the largest public component, which is why it supports both the main presence-side reversal analysis and the supporting absence-side analysis.
Panel~B shows that a small set of target finding families dominates the public protocols, confirming that the benchmark is clinically structured rather than arbitrary.
Panel~C summarizes image-count distribution and confirms that multi-image studies exist even though single-image cases dominate.
Panel~D compares protocol size with target-family breadth, and Panel~E makes the public-data substrate concrete through representative thumbnails from OpenI and CheXpert without exposing internal ReXVQA images.

\begin{figure*}[htbp]
  \centering
  \includegraphics[width=\linewidth]{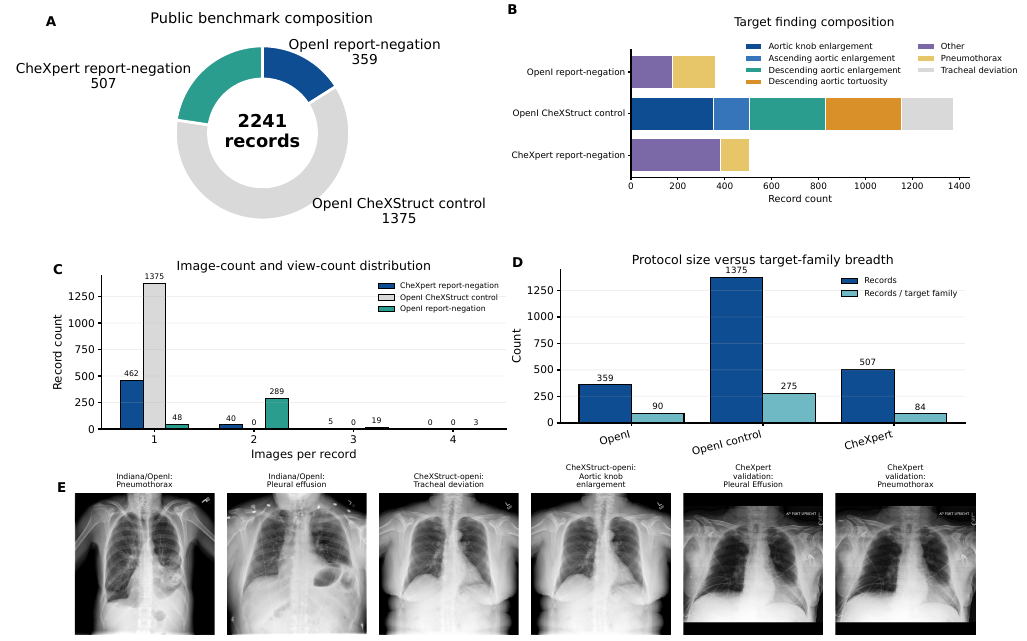}
  \caption{Benchmark composition atlas. The panels summarize public protocol size, target-finding composition, image-count distribution, and representative public thumbnails from OpenI and CheXpert. This figure documents the public benchmark substrate behind the main external analyses.}
  \label{fig:app-benchmark-composition}
\end{figure*}



\subsection{Verifier Logic and Thresholds}

Y0 replaces the baseline answer only when baseline confidence is below 0.80 and packet confidence exceeds it by at least 0.03. M1 is deterministic and requires an explicit negation cue in the question, exactly one negative option, baseline selection of that negative option, and a non-ambiguous option set. In the original protocol builders the clinically normalized target concept is placed in the first positive slot, but the slot-agnostic variant used for the layout analysis instead matches the positive concept directly and exactly reproduces the earlier slot-0 behavior on the tested original-layout CheXpert runs. The paraphrase robustness analysis keeps the same replacement logic but extends the trigger vocabulary beyond \emph{absent}. M2 is a location-conflict scaffold and is reported only as an ablation on the 1k probe.

\subsection{Model-Specific Execution Notes}

External multimodel runs share the same protocol JSONs and scoring code, but model loaders are adapted to each backbone. In our environment, CheXagent required \texttt{trust\_remote\_code} loading, explicit attention masks, \texttt{use\_cache=False}, and float32 inference for stability. CheXagent also preserves all study images, whereas LLaVA-1.5 and LLaVA-Med consume the first image in a study. CoT variants differ only in baseline prompting; M1 is applied to the resulting prediction without changing the verifier logic.

\subsection{Reported CoT Prompt Templates}

Because the CoT results are prompt-template-sensitive, Figure~\ref{fig:app-cot-prompt-templates} summarizes the fixed scaffolds used in the reported runs. The absence-side OpenI and CheXpert negation protocols use \texttt{absence\_step}; the direct CheXpert presence probe uses the related \texttt{polarity\_step} scaffold.

\begin{figure*}[htbp]
\centering

\begin{tcolorbox}[
  enhanced,
  colback=cotgray,
  colframe=cotline,
  boxrule=0.8pt,
  arc=2mm,
  left=2mm,right=2mm,top=1.5mm,bottom=1.5mm,
  width=0.96\linewidth
]
\small
\textbf{Shared CoT skeleton used in all reported CoT-B0 runs.}
Each prompt first presents the question and answer options, then forces an explicit polarity-disambiguation routine:
\begin{enumerate}[leftmargin=5mm,itemsep=1pt,topsep=2pt]
    \item identify whether the question asks about \textbf{absence} or \textbf{presence},
    \item identify the \textbf{finding concept itself} (without negation),
    \item identify the \textbf{negated surface form} (e.g., ``No $X$''),
    \item output a \textbf{single final answer letter}.
\end{enumerate}
\end{tcolorbox}

\vspace{1.5mm}

\begin{minipage}[t]{0.485\textwidth}
\begin{tcolorbox}[
  enhanced,
  equal height group=cotpair,
  valign=top,
  colback=absbluebg,
  colframe=absblue,
  boxrule=1pt,
  arc=2mm,
  title=\textbf{Absence-side scaffold (\texttt{absence\_step})},
  fonttitle=\bfseries,
  coltitle=black,
  left=2mm,right=2mm,top=1.2mm,bottom=1.2mm
]
\small
\textbf{Input format}
\begin{itemize}[leftmargin=4mm,itemsep=1pt,topsep=2pt]
    \item Question text
    \item A/B/C/D answer options
\end{itemize}

\vspace{0.5mm}
\textbf{Fixed reasoning steps}
\begin{enumerate}[leftmargin=5mm,itemsep=1pt,topsep=2pt]
    \item Does the question ask for an \textbf{ABSENT / NOT PRESENT} finding?
    \item Which option names the \textbf{finding concept itself}?
    \item Which option is the \textbf{negated surface form} (``No $X$'')?
\end{enumerate}

\vspace{0.5mm}
\begin{tcolorbox}[
  enhanced,
  colback=white,
  colframe=absblue,
  boxrule=0.7pt,
  arc=1.4mm,
  left=1.5mm,right=1.5mm,top=1mm,bottom=1mm
]
\small
\textbf{Decision rule}\\
If the question asks for \textbf{absence}, choose the \textbf{finding name}, \emph{not} the negated form.
\end{tcolorbox}

\vspace{0.5mm}
\textbf{Output}\\
Final answer: \texttt{[A/B/C/D]}
\end{tcolorbox}
\end{minipage}\hfill
\begin{minipage}[t]{0.485\textwidth}
\begin{tcolorbox}[
  enhanced,
  equal height group=cotpair,
  valign=top,
  colback=presgreenbg,
  colframe=presgreen,
  boxrule=1pt,
  arc=2mm,
  title=\textbf{Presence-side scaffold (\texttt{polarity\_step})},
  fonttitle=\bfseries,
  coltitle=black,
  left=2mm,right=2mm,top=1.2mm,bottom=1.2mm
]
\small
\textbf{Input format}
\begin{itemize}[leftmargin=4mm,itemsep=1pt,topsep=2pt]
    \item Question text
    \item A/B/C/D answer options
\end{itemize}

\vspace{0.5mm}
\textbf{Fixed reasoning steps}
\begin{enumerate}[leftmargin=5mm,itemsep=1pt,topsep=2pt]
    \item Does the question ask about \textbf{absence} or \textbf{presence}?
    \item Which option names the \textbf{finding concept itself}?
    \item Which option is the \textbf{negated surface form} (``No $X$'')?
\end{enumerate}

\vspace{0.5mm}
\begin{tcolorbox}[
  enhanced,
  colback=white,
  colframe=presgreen,
  boxrule=0.7pt,
  arc=1.4mm,
  left=1.5mm,right=1.5mm,top=1mm,bottom=1mm
]
\small
\textbf{Decision rule}\\
If the question asks for \textbf{absence}, choose the finding name rather than the negated form.\\
If the question asks for \textbf{presence}, choose the \textbf{present finding} and avoid the negated form.
\end{tcolorbox}

\vspace{0.5mm}
\textbf{Output}\\
Final answer: \texttt{[A/B/C/D]}
\end{tcolorbox}
\end{minipage}

\vspace{1mm}

\begin{tcolorbox}[
  enhanced,
  colback=warnorangebg,
  colframe=warnorange,
  boxrule=0.8pt,
  arc=2mm,
  left=2mm,right=2mm,top=1mm,bottom=1mm,
  width=0.96\linewidth
]
\small
\textbf{Key difference.}
Both prompts force explicit polarity disambiguation, but \texttt{polarity\_step} adds an explicit \textbf{presence-versus-absence branch} before the final answer, whereas \texttt{absence\_step} assumes an absence-targeted question from the outset.
\end{tcolorbox}

\caption{Compact schematic of the fixed hand-written CoT prompt scaffolds used in the reported CoT-B0 runs. Both templates share the same polarity-disambiguation skeleton; the presence-side scaffold differs mainly in its explicit branch between absence and presence questions.}
\label{fig:app-cot-prompt-templates}
\end{figure*}

\subsection{Uncertainty Estimates and Compute}

Statistical uncertainty is computed by \texttt{scripts/28\_bootstrap\_ci.py}, which performs 2,000 example-level bootstrap resamples with seed 42 for accuracy and contradiction counts and paired permutation tests for B0-vs.-M1 accuracy deltas. The reported experiments were executed in a local multi-GPU environment using a micromamba-managed Python stack. The 10k ReXVQA runs used 8 shards over 4 GPUs (2 processes per GPU); external protocols used 4 shards over 4 GPUs. The paper reports only inference-time experiments and post-hoc deterministic verification; no additional VLM training is introduced.

\section{Statistical Uncertainty of Main Results}
\label{app:stats}

Aggregate accuracy alone does not reveal whether reported B0-to-M1 gains are statistically reliable.
This section provides full bootstrap confidence intervals, paired permutation tests, and a study-level robustness check for the completed absence-side external comparisons and internal ReXVQA runs. The direct presence probes are reported as point estimates.

\subsection{Bootstrap Protocol}

For every main B0/M1 comparison, we estimate 95\% bootstrap confidence intervals from 2,000 resamples (seed~42). We additionally run paired permutation tests on example-level correctness to evaluate whether the B0-to-M1 accuracy delta could arise under the null of no improvement. Figure~\ref{fig:app-uncertainty-atlas} provides the full visual summary across all completed runs. Panel~A places all external B0-to-M1 deltas on the same forest plot; every completed external gain stays positive. Panel~B gives the corresponding internal ReXVQA forest plot, where gains remain small and often non-significant, consistent with the sparse-intervention characterization in the main text. Panel~C reports contradiction-count intervals and shows that the external contradiction burden is large and tightly estimated. Panel~D summarizes permutation-test $p$~values on a common scale, making the contrast between strongly significant external gains and modest internal gains visually explicit.

\begin{figure*}[htbp]
  \centering
  \includegraphics[width=\linewidth]{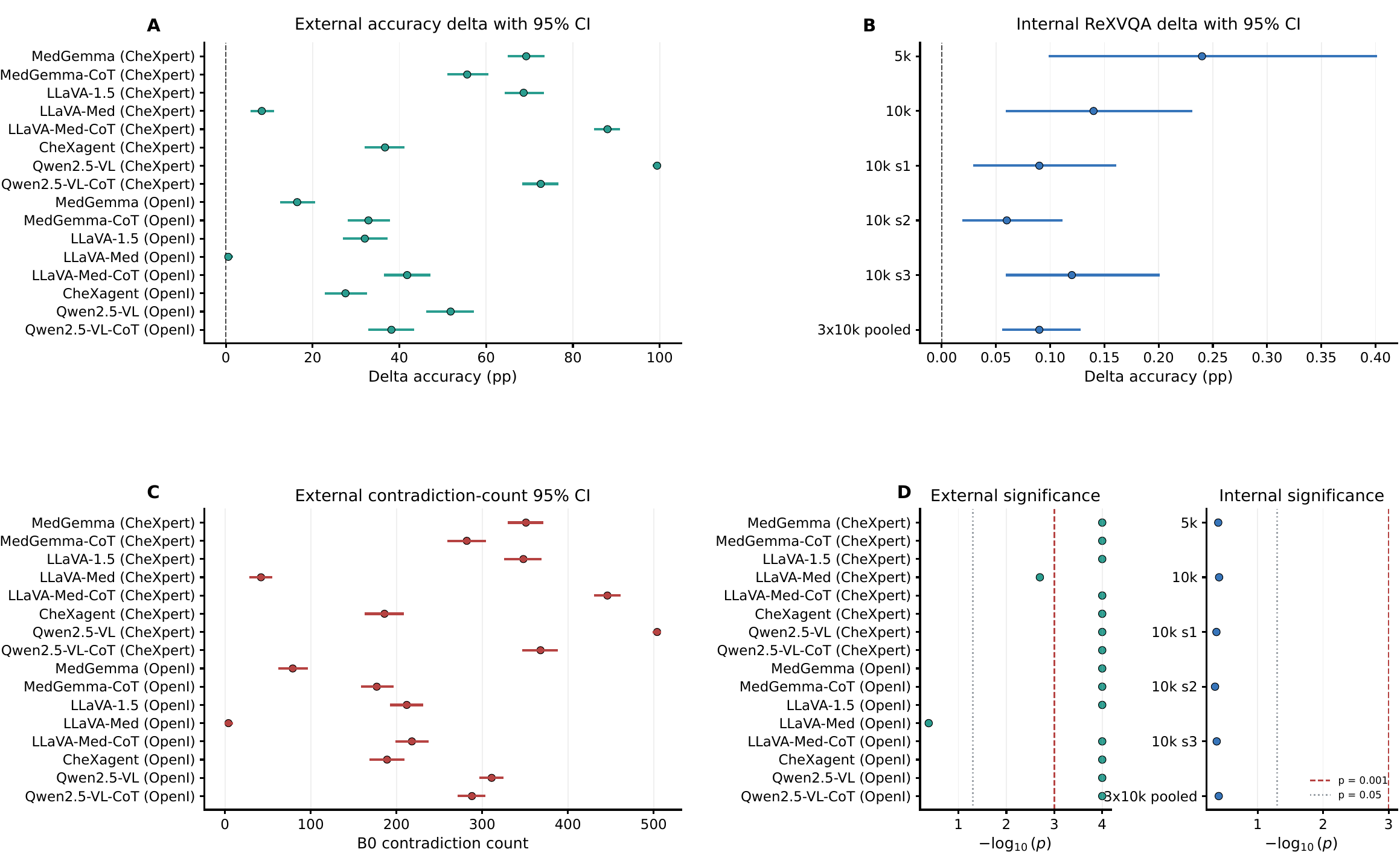}
  \caption{Statistical uncertainty atlas. The panels report external and internal 95\% bootstrap confidence intervals, external contradiction-count intervals, and permutation-test $p$ values for all completed B0-to-M1 comparisons.}
  \label{fig:app-uncertainty-atlas}
\end{figure*}

\subsection{Study-Level Robustness on CheXpert}

Because the CheXpert protocol contains 507~records drawn from 132~studies (mean 3.84~records per study), example-level resampling may underestimate confidence interval width by treating within-study records as independent.
To address this, we additionally report study-clustered bootstrap CIs that resample at the study level: in each of 2,000~iterations, 132~studies are drawn with replacement and all their records are included (Table~\ref{tab:study_bootstrap}).
Study-level delta CIs are slightly wider than example-level CIs (median width 8.18~pp vs.\ 7.79~pp), yet all eight CheXpert variants remain significant under study-level permutation tests ($p < 0.001$ throughout), with no change in significance status relative to Table~\ref{tab:external-stats}. The main conclusions therefore survive a stricter resampling scheme that respects within-study correlation.

\begin{table}[h]
\centering
\caption{Example-level vs.\ study-level bootstrap 95\% CIs for the B0-to-M1 accuracy delta on CheXpert ($n{=}507$ records, 132~studies). Study-level resampling accounts for within-study record correlation. All variants remain significant under both resampling schemes.}
\label{tab:study_bootstrap}
\small
\begin{tabular}{lrrrrrr}
\toprule
Variant &
\multicolumn{2}{c}{Example-level CI} & $p$ &
\multicolumn{2}{c}{Study-level CI} & $p$ \\
\cmidrule(lr){2-3} \cmidrule(lr){5-6}
& low & high & & low & high & \\
\midrule
MedGemma & 65.28 & 73.18 & $<0.001$ & 64.59 & 73.68 & $<0.001$ \\
MedGemma-CoT & 51.28 & 60.16 & $<0.001$ & 51.36 & 59.81 & $<0.001$ \\
LLaVA-1.5 & 64.50 & 72.98 & $<0.001$ & 64.12 & 73.17 & $<0.001$ \\
LLaVA-Med & 5.92 & 10.85 & 0.002 & 5.79 & 10.89 & $<0.001$ \\
LLaVA-Med-CoT & 85.21 & 90.53 & $<0.001$ & 85.29 & 90.70 & $<0.001$ \\
CheXagent & 32.35 & 40.83 & $<0.001$ & 30.15 & 43.60 & $<0.001$ \\
Qwen2.5-VL & 98.62 & 100.00 & $<0.001$ & 98.66 & 100.00 & $<0.001$ \\
Qwen2.5-VL-CoT & 68.64 & 76.33 & $<0.001$ & 68.51 & 76.42 & $<0.001$ \\
\bottomrule
\end{tabular}
\end{table}

\subsection{Full External Confidence Intervals}

Table~\ref{tab:external-stats} reports exact accuracy and contradiction-count intervals for every completed external comparison, and serves as the precise statistical backstop for the main external claim. Fourteen of the sixteen B0-to-M1 gains satisfy $p<0.001$; the two exceptions are the non-CoT LLaVA-Med rows, consistent with that model's unusually low baseline contradiction count on these protocols.

\begin{table*}[htbp]
\caption{Bootstrap confidence intervals and paired permutation tests for all completed external absence comparisons. Fourteen of the sixteen B0-to-M1 gains satisfy $p<0.001$; the two exceptions are the non-CoT LLaVA-Med rows.}
\label{tab:external-stats}
\centering
\scriptsize
\resizebox{\textwidth}{!}{%
\begin{tabular}{llcccc}
\toprule
Variant & Protocol & B0 Acc. [95\% CI] & M1 Acc. [95\% CI] & Contrad. $B0$ [95\% CI] $\rightarrow M1$ & Perm. $p$ \\
\midrule
MedGemma & CheXpert & 24.06 [20.51, 27.81] & 93.29 [91.12, 95.47] & 351 [331, 370] $\rightarrow$ 0 & $<0.001$ \\
MedGemma & OpenI & 70.47 [65.74, 75.21] & 86.91 [83.57, 90.25] & 79 [64, 95] $\rightarrow$ 0 & $<0.001$ \\
MedGemma-CoT & CheXpert & 32.54 [28.60, 36.69] & 88.17 [85.40, 90.93] & 282 [261, 303] $\rightarrow$ 0 & $<0.001$ \\
MedGemma-CoT & OpenI & 45.68 [40.39, 50.70] & 78.55 [74.09, 82.73] & 177 [160, 195] $\rightarrow$ 0 & $<0.001$ \\
LLaVA-1.5 & CheXpert & 27.42 [23.47, 31.36] & 96.06 [94.28, 97.64] & 348 [327, 368] $\rightarrow$ 0 & $<0.001$ \\
LLaVA-1.5 & OpenI & 40.39 [35.38, 45.40] & 72.42 [67.97, 76.88] & 212 [194, 230] $\rightarrow$ 0 & $<0.001$ \\
LLaVA-Med & CheXpert & 67.06 [63.31, 71.20] & 75.35 [71.60, 79.09] & 42 [30, 54] $\rightarrow$ 0 & 0.002 \\
LLaVA-Med & OpenI & 49.03 [43.73, 54.04] & 49.58 [44.57, 54.87] & 4 [1, 8] $\rightarrow$ 0 & 0.417 \\
LLaVA-Med-CoT & CheXpert & 11.83 [9.07, 14.79] & 99.80 [99.41, 100.00] & 446 [432, 460] $\rightarrow$ 0 & $<0.001$ \\
LLaVA-Med-CoT & OpenI & 38.72 [33.70, 43.73] & 80.50 [76.32, 84.40] & 218 [200, 236] $\rightarrow$ 0 & $<0.001$ \\
CheXagent & CheXpert & 6.51 [4.54, 8.68] & 43.20 [39.05, 47.53] & 186 [164, 207] $\rightarrow$ 0 & $<0.001$ \\
CheXagent & OpenI & 2.51 [1.11, 4.18] & 30.08 [25.35, 34.83] & 189 [170, 208] $\rightarrow$ 0 & $<0.001$ \\
Qwen2.5-VL & CheXpert & 0.39 [0.00, 0.99] & 99.80 [99.41, 100.00] & 504 [500, 507] $\rightarrow$ 0 & $<0.001$ \\
Qwen2.5-VL & OpenI & 12.53 [9.19, 15.88] & 64.35 [59.33, 69.08] & 311 [298, 323] $\rightarrow$ 0 & $<0.001$ \\
Qwen2.5-VL-CoT & CheXpert & 27.22 [23.47, 31.16] & 99.80 [99.41, 100.00] & 368 [348, 387] $\rightarrow$ 0 & $<0.001$ \\
Qwen2.5-VL-CoT & OpenI & 14.76 [11.42, 18.38] & 52.92 [47.91, 57.94] & 288 [273, 302] $\rightarrow$ 0 & $<0.001$ \\
\bottomrule
\end{tabular}
\vphantom{A}}
\end{table*}

\subsection{Internal MedGemma Confidence Intervals}

In contrast to the external regime, internal ReXVQA gains are consistently positive but remain small and non-significant by permutation test (Table~\ref{tab:internal-stats}). This pattern is important for honest reporting: internal ReXVQA supports the sparse-intervention claim but not a large aggregate-accuracy claim. The effect sizes in Table~\ref{tab:internal-stats} therefore serve as an upper bound on what the verifier contributes in a high-accuracy in-distribution setting.

\begin{table}[htbp]
\caption{Bootstrap confidence intervals and paired permutation tests for completed internal MedGemma ReXVQA comparisons. The gains are consistently positive but remain small and non-significant by permutation test.}
\label{tab:internal-stats}
\centering
\scriptsize
\resizebox{\linewidth}{!}{%
\begin{tabular}{lcccc}
\toprule
Setting & B0 Acc. [95\% CI] & M1 Acc. [95\% CI] & $\Delta$ pp [95\% CI] & Perm. $p$ \\
\midrule
ReXVQA 5k & 82.42 [81.40, 83.46] & 82.66 [81.60, 83.72] & 0.24 [0.10, 0.40] & 0.399 \\
ReXVQA 10k & 81.28 [80.55, 82.07] & 81.42 [80.68, 82.21] & 0.14 [0.06, 0.23] & 0.387 \\
ReXVQA 10k s1 & 80.53 [79.75, 81.32] & 80.62 [79.87, 81.38] & 0.09 [0.03, 0.16] & 0.424 \\
ReXVQA 10k s2 & 80.34 [79.55, 81.12] & 80.40 [79.62, 81.13] & 0.06 [0.02, 0.11] & 0.446 \\
ReXVQA 10k s3 & 80.21 [79.44, 80.98] & 80.33 [79.59, 81.14] & 0.12 [0.06, 0.20] & 0.420 \\
ReXVQA pooled & 80.36 [79.92, 80.80] & 80.45 [79.98, 80.90] & 0.09 [0.06, 0.13] & 0.392 \\
\bottomrule
\end{tabular}
\vphantom{A}}
\end{table}

\section{Additional Internal Stability Results}
\label{app:stability}

The main text summarizes internal ReXVQA results compactly. This section provides the full supporting evidence for sparsity, seeded reproducibility, and the ablation comparisons motivating the choice of M1 over Y0 and M2.

\subsection{Main Internal Results}

\begin{table*}[htbp]
\caption{Supporting internal ReXVQA results. The 1k probe spans five backbones; the longer 5k, 10k, and pooled runs remain MedGemma-only. On the 1k probe, intervention coverage stays below 1\% and worsened remains 0 for every tested backbone.}
\label{tab:internal-results}
\centering
\small
\resizebox{\textwidth}{!}{%
\begin{tabular}{llccccccccc}
\toprule
Setting & Model & B0 Acc. & M1 Acc. & $\Delta$ pp & Changed & Improved & Worsened & Explicit-neg. $\Delta$ pp & Cov. \% & Contradictions \\
\midrule
1k probe & MedGemma & 82.10 & 82.50 & +0.40 & 4 & 4 & 0 & +8.51 & 0.40 & 4 $\rightarrow$ 0 \\
1k probe & LLaVA-1.5 & 34.40 & 35.00 & +0.60 & 6 & 6 & 0 & +12.77 & 0.60 & 6 $\rightarrow$ 0 \\
1k probe & LLaVA-Med & 27.30 & 27.30 & +0.00 & 1 & 0 & 0 & +0.00 & 0.10 & 1 $\rightarrow$ 0 \\
1k probe & Qwen2.5-VL & 67.90 & 68.50 & +0.60 & 6 & 6 & 0 & +12.76 & 0.60 & 6 $\rightarrow$ 0 \\
1k probe & CheXagent & 52.70 & 53.20 & +0.50 & 6 & 5 & 0 & +10.64 & 0.60 & 6 $\rightarrow$ 0 \\
\midrule
5k slice & MedGemma & 82.42 & 82.66 & +0.24 & 17 & 13 & 1 & +5.69 & 0.34 & 17 $\rightarrow$ 0 \\
\midrule
10k slice & MedGemma & 81.28 & 81.42 & +0.14 & 24 & 16 & 2 & +3.23 & 0.24 & 24 $\rightarrow$ 0 \\
\midrule
3x10k pooled & MedGemma & 80.36 & 80.45 & +0.09 & 51 & 28 & 1 & +1.99 & 0.17 & 51 $\rightarrow$ 0 \\
\bottomrule
\end{tabular}
\vphantom{A}}
\end{table*}

\subsection{Intervention Sparsity and Scale Stability}

A key property of \method{} is that it should remain sparse and stable as scale increases. Figure~\ref{fig:app-internal-sparsity-atlas} consolidates this evidence across four panels. Intervention coverage stays well below 1\% while the accuracy delta remains positive across 1k, 5k, 10k, and pooled settings~(Panel~A). B0 and M1 track each other closely as scale grows, indicating that the verifier does not produce unstable divergence at longer horizons~(Panel~B). The three seeded 10k runs remain consistently but modestly positive~(Panel~C). Panel~D compares ablations and shows that Y0 underperforms while M2 adds no meaningful gain over M1.

\begin{figure*}[htbp]
  \centering
  \includegraphics[width=\linewidth]{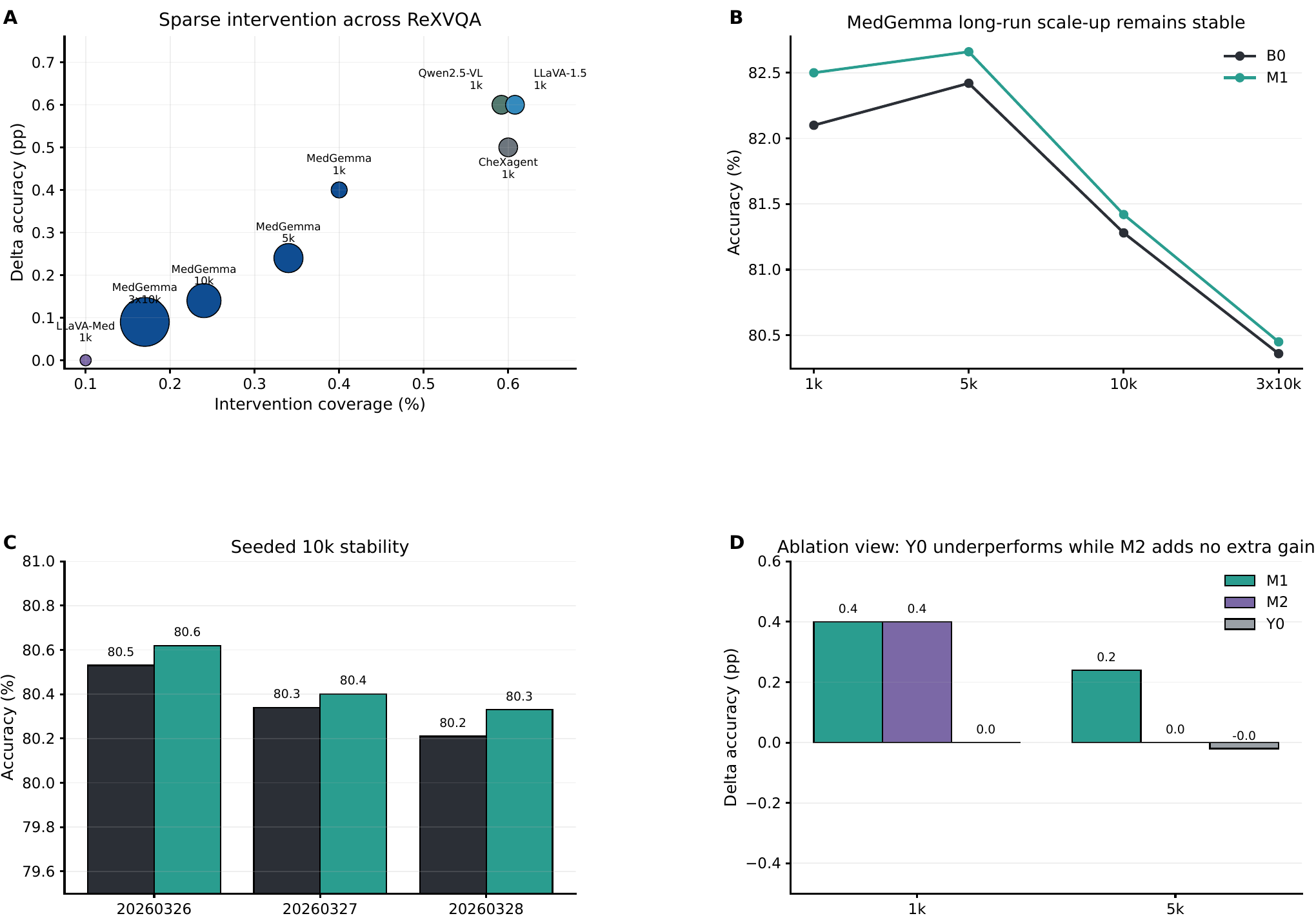}
  \caption{Internal sparsity atlas. The panels summarize intervention coverage, long-run MedGemma scale-up, seeded 10k stability, and ablation behavior for Y0, M1, and M2. The main pattern remains sparse intervention with small but positive aggregate gains.}
  \label{fig:app-internal-sparsity-atlas}
\end{figure*}

\subsection{Ablation Comparisons: Y0 and M2}

Table~\ref{tab:internal-ablations} isolates the two MedGemma-only ablations referenced in Section~\ref{sec:qccv}. Y0, the packet-confidence fallback, changes 74 predictions on the 5k slice, more than an order of magnitude above M1, yet slightly \emph{decreases} overall accuracy and leaves contradiction counts entirely unchanged, confirming that broad confidence-based overriding is not a substitute for targeted polarity correction. M2, the location scaffold, merely matches M1 on the 1k probe and offers no additional benefit. These results explain why the main paper foregrounds M1 rather than either ablation.

\begin{table}[h]
\caption{MedGemma-only internal ablations omitted from Table~\ref{tab:internal-results}. Y0 is the packet-only confidence fallback; M2 is the location scaffold evaluated only on the 1k probe. Neither changes the paper's main conclusion.}
\label{tab:internal-ablations}
\centering
\small
\resizebox{\linewidth}{!}{%
\begin{tabular}{llccccccc}
\toprule
Setting & Variant & Acc. & $\Delta$ pp & Changed & Improved & Worsened & Explicit-neg. $\Delta$ pp & Contradictions \\
\midrule
1k probe & M2 & 82.50 & +0.40 & 4 & 4 & 0 & +8.50 & 4 $\rightarrow$ 0 \\
5k slice & Y0 & 82.40 & -0.02 & 74 & 29 & 30 & -0.95 & 17 $\rightarrow$ 17 \\
\bottomrule
\end{tabular}
\vphantom{A}}
\end{table}

\subsection{Seeded Reproducibility}

To verify that the pooled 30k result is not a single-sample accident, Table~\ref{tab:seeded-stability} reports three independent seeded 10k runs. All three seeds show the same qualitative pattern: a small positive accuracy delta and complete removal of measured explicit-negation contradictions, with no change in the overall sign or significance of the effect.

\begin{table}[h]
\caption{Seeded ReXVQA 10k stability runs for B0 vs.\ M1. All three seeds show positive net gain and complete removal of measured explicit-negation contradictions.}
\label{tab:seeded-stability}
\centering
\small
\begin{tabular}{lcccc}
\toprule
Seed & B0 Acc. & M1 Acc. & $\Delta$ pp & Contradictions \\
\midrule
20260326 & 80.53 & 80.62 & +0.09 & 15 $\rightarrow$ 0 \\
20260327 & 80.34 & 80.40 & +0.06 & 15 $\rightarrow$ 0 \\
20260328 & 80.21 & 80.33 & +0.12 & 21 $\rightarrow$ 0 \\
\bottomrule
\end{tabular}
\end{table}

\section{Additional External Analyses}
\label{app:external}

This section supplements the main external results with five analyses: the full model-by-protocol result surface, a large-scale training-split replication, a retrospective OpenI audit, robustness stress tests under paraphrase and layout perturbation, and a structured error breakdown by finding category and case audit.

\subsection{Full External Result Matrix}

Table~\ref{tab:external-results-full} reports the full numerical matrix behind the main external summary in Figure~\ref{fig:external-results}. We use \emph{Base} to denote the source model output: B0 for non-CoT rows and CoT-B0 for CoT rows. We use \emph{Verified} to denote the corresponding deterministic correction: M1 for non-CoT rows and CoT+M1 for CoT rows. For CheXpert direct presence, the polarity-error column counts dangerous ``No $X$'' reversals; for the two absence protocols, it counts contradictions.

\begin{table*}[htbp]
\caption{Full cross-model external result matrix. \emph{Base} is B0 for non-CoT rows and CoT-B0 for CoT rows; \emph{Verified} is M1 for non-CoT rows and CoT+M1 for CoT rows. For CheXpert direct presence, polarity errors are dangerous ``No $X$'' reversals; for CheXpert and OpenI absence protocols, they are contradictions.}
\label{tab:external-results-full}
\centering
\scriptsize
\resizebox{\textwidth}{!}{%
\begin{tabular}{lcccccccccccc}
\toprule
\multirow{2}{*}{Variant} 
& \multicolumn{4}{c}{CheXpert direct presence ($n=235$)} 
& \multicolumn{4}{c}{CheXpert report-style absence ($n=507$)} 
& \multicolumn{4}{c}{OpenI report-style absence ($n=359$)} \\
\cmidrule(lr){2-5}\cmidrule(lr){6-9}\cmidrule(lr){10-13}
& Base & Verified & $\Delta$ pp & \shortstack{Errors\\Base $\rightarrow$ Verified}
& Base & Verified & $\Delta$ pp & \shortstack{Errors\\Base $\rightarrow$ Verified}
& Base & Verified & $\Delta$ pp & \shortstack{Errors\\Base $\rightarrow$ Verified} \\
\midrule
MedGemma 
& 31.49 & 96.60 & +65.11 & 153 $\rightarrow$ 0 
& 24.06 & 93.29 & +69.23 & 351 $\rightarrow$ 0 
& 70.47 & 86.91 & +16.44 & 79 $\rightarrow$ 0 \\

MedGemma-CoT 
& 42.98 & 94.89 & +51.91 & 122 $\rightarrow$ 0 
& 32.54 & 88.17 & +55.63 & 282 $\rightarrow$ 0 
& 45.68 & 78.55 & +32.87 & 177 $\rightarrow$ 0 \\

MedGemma-27B 
& 51.91 & 96.60 & +44.68 & 105 $\rightarrow$ 0 
& 31.16 & 97.24 & +66.08 & 335 $\rightarrow$ 0 
& 15.32 & 52.37 & +37.05 & 199 $\rightarrow$ 0 \\

LLaVA-1.5 
& 56.60 & 68.09 & +11.49 & 27 $\rightarrow$ 0 
& 27.42 & 96.06 & +68.64 & 348 $\rightarrow$ 0 
& 40.39 & 72.42 & +32.03 & 212 $\rightarrow$ 0 \\

LLaVA-Med 
& 74.04 & 74.04 & +0.00 & 0 $\rightarrow$ 0 
& 67.06 & 75.35 & +8.29 & 42 $\rightarrow$ 0 
& 49.03 & 49.58 & +0.55 & 4 $\rightarrow$ 0 \\

LLaVA-Med-CoT 
& 68.09 & 68.09 & +0.00 & 0 $\rightarrow$ 0 
& 11.83 & 99.80 & +87.97 & 446 $\rightarrow$ 0 
& 38.72 & 80.50 & +41.78 & 218 $\rightarrow$ 0 \\

CheXagent 
& 37.87 & 40.85 & +2.98 & 7 $\rightarrow$ 0 
& 6.51 & 43.20 & +36.69 & 186 $\rightarrow$ 0 
& 2.51 & 30.08 & +27.57 & 189 $\rightarrow$ 0 \\

Qwen2.5-VL 
& 30.21 & 95.32 & +65.11 & 153 $\rightarrow$ 0 
& 0.39 & 99.80 & +99.41 & 504 $\rightarrow$ 0 
& 12.53 & 64.35 & +51.82 & 311 $\rightarrow$ 0 \\

Qwen2.5-VL-CoT 
& 54.47 & 78.72 & +24.26 & 57 $\rightarrow$ 0 
& 27.22 & 99.80 & +72.58 & 368 $\rightarrow$ 0 
& 14.76 & 52.92 & +38.16 & 288 $\rightarrow$ 0 \\

GPT-4o 
& 62.55 & 73.19 & +10.64 & 25 $\rightarrow$ 0 
& 2.17 & 94.87 & +92.70 & 470 $\rightarrow$ 0 
& 16.16 & 63.79 & +47.63 & 284 $\rightarrow$ 0 \\
\bottomrule
\end{tabular}
\vphantom{A}}
\end{table*}

Figure~\ref{fig:app-external-full-matrix} provides the complementary heatmap view for the completed absence-side external protocols, where uncertainty estimates and all model-by-protocol cells are fully tabulated. Baseline accuracy varies widely across model families and protocols~(Panel~A), so a single-model anecdote would be misleading; Panel~B shows the corresponding verified-output surface. The largest gains occur where the baseline failure is strongest~(Panel~C), and contradiction rate is high across much of the external matrix~(Panel~D), confirming a systematic failure rather than a narrow outlier. Panel~E shows that the fraction of predictions changed by M1 tracks the failure surface rather than behaving like random rewriting.

\begin{figure*}[htbp]
  \centering
  \includegraphics[width=\linewidth]{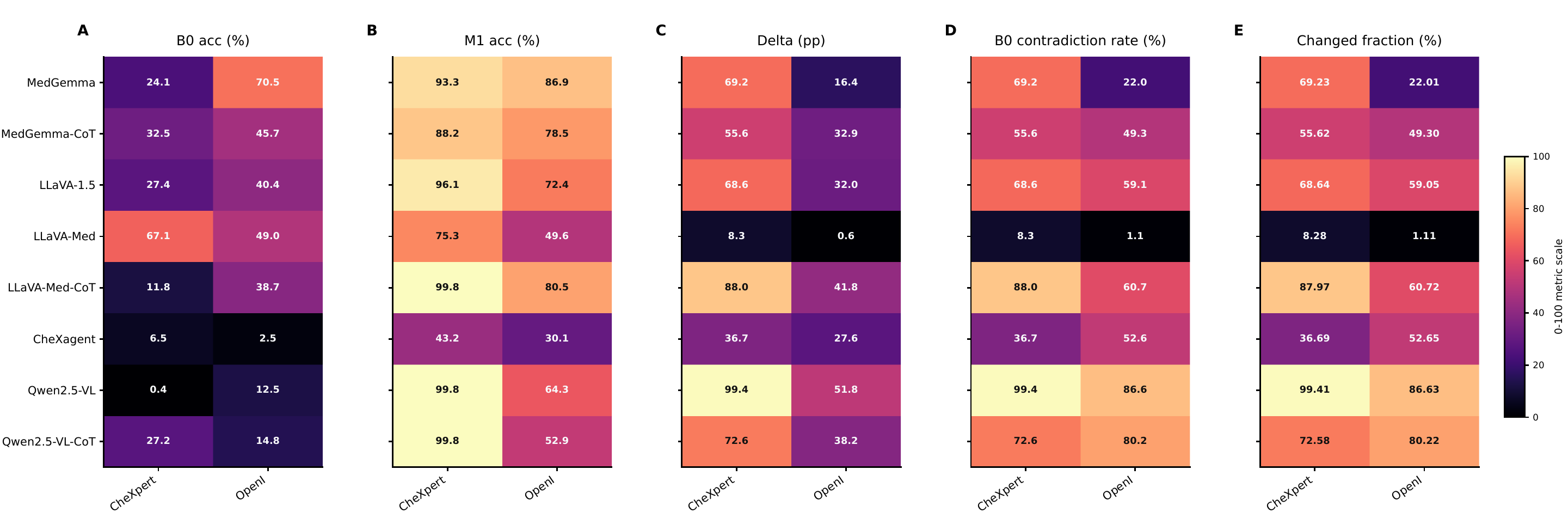}
  \caption{Full external result matrix across completed model and prompt variants on the absence-side CheXpert and OpenI protocols. Heatmaps report Base accuracy, Verified accuracy, accuracy delta, Base contradiction rate, and changed fraction.}
  \label{fig:app-external-full-matrix}
\end{figure*}

\subsection{Scale Replication on the CheXpert Training Split}

A potential concern is that the direct 235-record presence protocol and the 507-record absence validation protocol are both too small to rule out dataset-specific artifacts. We therefore replicate both analyses on matched CheXpert training-split protocols. The new direct presence replication yields 135{,}754 records from 59{,}173 studies, while the report-style absence replication yields 25{,}900 records from 15{,}119 studies. Table~\ref{tab:chexpert-train-scale} shows that the presence-side semantic reversal remains severe at scale: MedGemma and Qwen2.5-VL still emit 84{,}966 and 88{,}668 dangerous negated-option reversals, respectively, before deterministic repair. Figure~\ref{fig:app-scale-replication} retains the earlier absence-side visualization and still shows that the same answer-space mechanism remains strong on the larger report-style protocol.

\begin{figure*}[htbp]
  \centering
  \includegraphics[width=\linewidth]{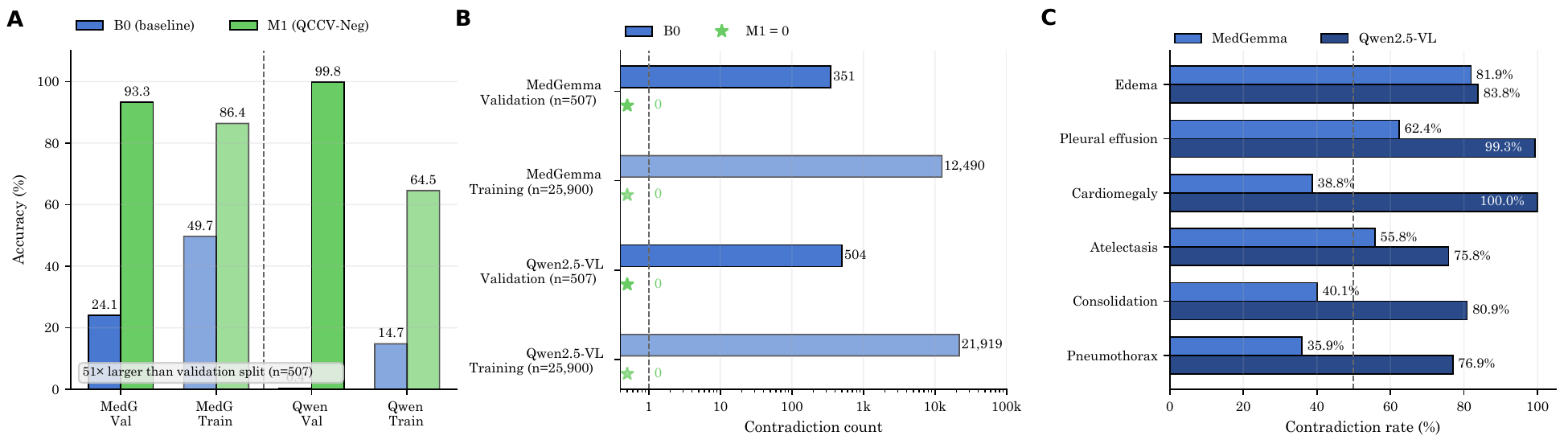}
  \caption{\textbf{Absence-side scale replication on the CheXpert training split.} \textbf{A}: B0 and M1 accuracy across validation ($n{=}507$) and training ($n{=}25{,}900$) splits for MedGemma and Qwen2.5-VL. The B0-to-M1 accuracy gain replicates at 51$\times$ scale. \textbf{B}: Contradiction counts on log scale. \method{} eliminates all 12{,}490 MedGemma and 21{,}919 Qwen contradictions on the training split, matching the zero-contradiction result on the validation split. \textbf{C}: Finding-level contradiction rates on the training split. The pathology pattern is consistent with the validation-split analysis in Table~\ref{tab:finding-breakdown}, confirming that the phenomenon is not a validation-set artifact.}
  \label{fig:app-scale-replication}
\end{figure*}

\begin{table}[h]
\caption{CheXpert training-split scale replication for both the new direct presence protocol and the report-style absence protocol. The direct presence replication is much larger and shows that dangerous ``No $X$'' reversals remain common at scale, while the absence-side replication preserves the earlier contradiction-heavy pattern.}
\label{tab:chexpert-train-scale}
\centering
\scriptsize
\resizebox{\linewidth}{!}{%
\begin{tabular}{llrrrrr}
\toprule
Protocol & Model & $n$ & B0 & M1 & ``No $X$''/Contrad. $B0 \rightarrow M1$ & Rate \\
\midrule
Direct presence val. & MedGemma & 235 & 31.49 & 96.60 & 153 $\rightarrow$ 0 & 65.11 \\
Direct presence train & MedGemma & 135{,}754 & 33.05 & 95.63 & 84{,}966 $\rightarrow$ 0 & 62.59 \\
Direct presence val. & Qwen2.5-VL & 235 & 30.21 & 95.32 & 153 $\rightarrow$ 0 & 65.11 \\
Direct presence train & Qwen2.5-VL & 135{,}754 & 30.47 & 95.78 & 88{,}668 $\rightarrow$ 0 & 65.32 \\
\midrule
Report-style absence val. & MedGemma & 507 & 24.06 & 93.29 & 351 $\rightarrow$ 0 & 69.23 \\
Report-style absence train & MedGemma & 25{,}900 & 49.67 & 86.37 & 12{,}490 $\rightarrow$ 0 & 48.22 \\
Report-style absence val. & Qwen2.5-VL & 507 & 0.39 & 99.80 & 504 $\rightarrow$ 0 & 99.41 \\
Report-style absence train & Qwen2.5-VL & 25{,}900 & 14.73 & 64.54 & 21{,}919 $\rightarrow$ 0 & 84.63 \\
\bottomrule
\end{tabular}
\vphantom{A}}
\end{table}

\subsection{Retrospective OpenI Audit}

Table~\ref{tab:openi-control} serves as a compact retrospective audit rather than a null control. On OpenI CheXStruct, MedGemma improves on both subsets, but the semantic interpretation differs. The absence subset retains the auditable contradiction pattern, while the presence subset reveals the clinically risky presence-side semantic reversal. Because both subsets share source distribution, task format, and evaluation pipeline, the comparison localizes the new gain to polarity handling rather than to any broad domain shift. The multimodel presence audit sharpens that picture: the negated-option wrong-answer fraction is 96.77\% for MedGemma, 92.58\% for Qwen2.5-VL, 41.61\% for LLaVA-1.5, and 21.19\% for LLaVA-Med, which again shows strong backbone dependence rather than a universal effect.

\begin{table}[h]
\caption{Retrospective OpenI CheXStruct audit. Upper block: MedGemma all/absence/presence split (B0 and M1). Here ``Wrong with No $X$'' counts all incorrect predictions whose selected option is the negated form; on the presence subset this quantity is broader than the narrower subset of repairable reversals. Lower block: multimodel presence-side negated-option wrong-answer rates on B0, confirming strong backbone dependence. Accuracy figures across the two blocks are not directly comparable due to different finding distributions.}
\label{tab:openi-control}
\centering
\small
\begin{tabular}{lcccc}
\toprule
Split & $n$ & B0 Acc. & M1 Acc. & Wrong with ``No $X$'' $B0 \rightarrow M1$ \\
\midrule
All & 1375 & 22.04 & 51.27 & 1054 $\rightarrow$ 501 \\
Absence & 979 & 19.00 & 39.94 & 784 $\rightarrow$ 428 \\
Presence & 396 & 29.55 & 79.29 & 270 $\rightarrow$ 73 \\
\bottomrule
\end{tabular}
\vspace{0.5em}

{\footnotesize\emph{Note:} On the MedGemma presence row, $270 \rightarrow 73$ counts all wrong ``No $X$'' predictions. The narrower repairable-reversal subset is $197$, which is exactly the number corrected by \method{} on this retrospective audit.}

\vspace{0.5em}

\begin{tabular}{lrrrr}
\toprule
Model & $n$ & B0 Acc.\ (\%) & Wrong ``No~$X$'' (B0) & Rate (\%) \\
\midrule
MedGemma & 396 & 29.55 & 270 & 96.77 \\
Qwen2.5-VL & 396 & 21.72 & 287 & 92.58 \\
LLaVA-1.5 & 396 & 62.37 & 62 & 41.61 \\
LLaVA-Med & 396 & 70.20 & 25 & 21.19 \\
\bottomrule
\end{tabular}
\end{table}

\subsection{Paraphrase and Layout Robustness}

The main text reports a compact CheXpert absence-side paraphrase summary in Table~\ref{tab:main-paraphrase-robustness}. Here we provide the full numerical breakdowns and the corresponding visual atlas for paraphrase and layout perturbations.

\paragraph{Paraphrase robustness.}
Figure~\ref{fig:contradictions} gives the main lexical-variation view. The broader non-canonical negation-paraphrase stress test shows that B0 polarity confusion persists beyond the literal ``No $X$'' form. The original literal-surface verifier loses coverage on several paraphrases, whereas a paraphrase-aware deterministic extension restores the measurable subset with no worsened cases in this experiment.

\begin{table}[h]
\caption{Full paraphrase-robustness breakdown on CheXpert report-style negation. The original M1 trigger misses P1 and P3, whereas the extended-trigger verifier removes all measured contradictions on original, P1, P2, and P3 with worsened = 0 throughout.}
\label{tab:paraphrase-robustness-full}
\centering
\scriptsize
\resizebox{\linewidth}{!}{%
\begin{tabular}{lcccccccc}
\toprule
Variant & $n$ & B0 Acc. & B0 Contrad. & Orig. M1 Acc. & Orig. M1 Contrad. & Ext. M1 Acc. & Ext. M1 Contrad. & Ext. M1 Worsened \\
\midrule
Original & 507 & 24.06 & 351 & 93.29 & 0 & 93.29 & 0 & 0 \\
P1 & 507 & 22.68 & 283 & 22.68 & 283 & 78.50 & 0 & 0 \\
P2 & 507 & 18.93 & 381 & 94.08 & 0 & 94.08 & 0 & 0 \\
P3 & 507 & 24.65 & 379 & 24.65 & 379 & 99.41 & 0 & 0 \\
\bottomrule
\end{tabular}
\vphantom{A}}
\end{table}
\paragraph{Layout robustness.}
The original protocol places the target concept in a fixed answer slot. To test whether M1 depends on this layout, we generate shuffled-order variants of the CheXpert protocol and compare slot-0 M1 against a slot-agnostic variant that matches the positive concept by value rather than by position. Table~\ref{tab:layout-robustness} shows that on the original layout, slot-agnostic M1 exactly matches slot-0 M1 for both MedGemma and Qwen2.5-VL, so the generalized rule introduces no regression. On shuffled layouts, however, slot-0 M1 accuracy drops to 84.42\% (MedGemma) and 67.06\% (Qwen2.5-VL), while slot-agnostic M1 remains at 97.44\% and 99.41\% with contradictions still at zero. The slot-agnostic rule therefore closes the fixed-slot loophole without altering the canonical correction behavior.

\begin{table}[h]
\caption{Layout robustness on CheXpert report-style negation. Slot-agnostic M1 matches slot-0 exactly on the original layout for MedGemma and Qwen, but remains substantially stronger after shuffling the option order.}
\label{tab:layout-robustness}
\centering
\scriptsize
\resizebox{\linewidth}{!}{%
\begin{tabular}{llccccccc}
\toprule
Layout & Model & $n$ & B0 Acc. & B0 Contrad. & Slot-0 M1 Acc. & Slot-0 M1 Contrad. & Slot-agnostic M1 Acc. & Slot-agnostic M1 Contrad. \\
\midrule
Original & MedGemma & 507 & 24.06 & 351 & 93.29 & 0 & 93.29 & 0 \\
Shuffled & MedGemma & 507 & 46.94 & 256 & 84.42 & 0 & 97.44 & 0 \\
Original & Qwen & 507 & 0.39 & 504 & 99.80 & 0 & 99.80 & 0 \\
Shuffled & Qwen & 507 & 16.57 & 420 & 67.06 & 0 & 99.41 & 0 \\
\bottomrule
\end{tabular}
\vphantom{A}}
\end{table}

Figure~\ref{fig:app-robustness-atlas} visualizes both robustness analyses jointly. The top row reports accuracy and contradiction behavior under paraphrase variants with the original and extended triggers; the bottom row reports the layout comparison. Both robustness issues are real, but they are addressed by different mechanism changes, namely extended trigger coverage for paraphrase and slot-agnostic replacement for layout, without interfering with each other.

\begin{figure*}[htbp]
  \centering
  \includegraphics[width=\linewidth]{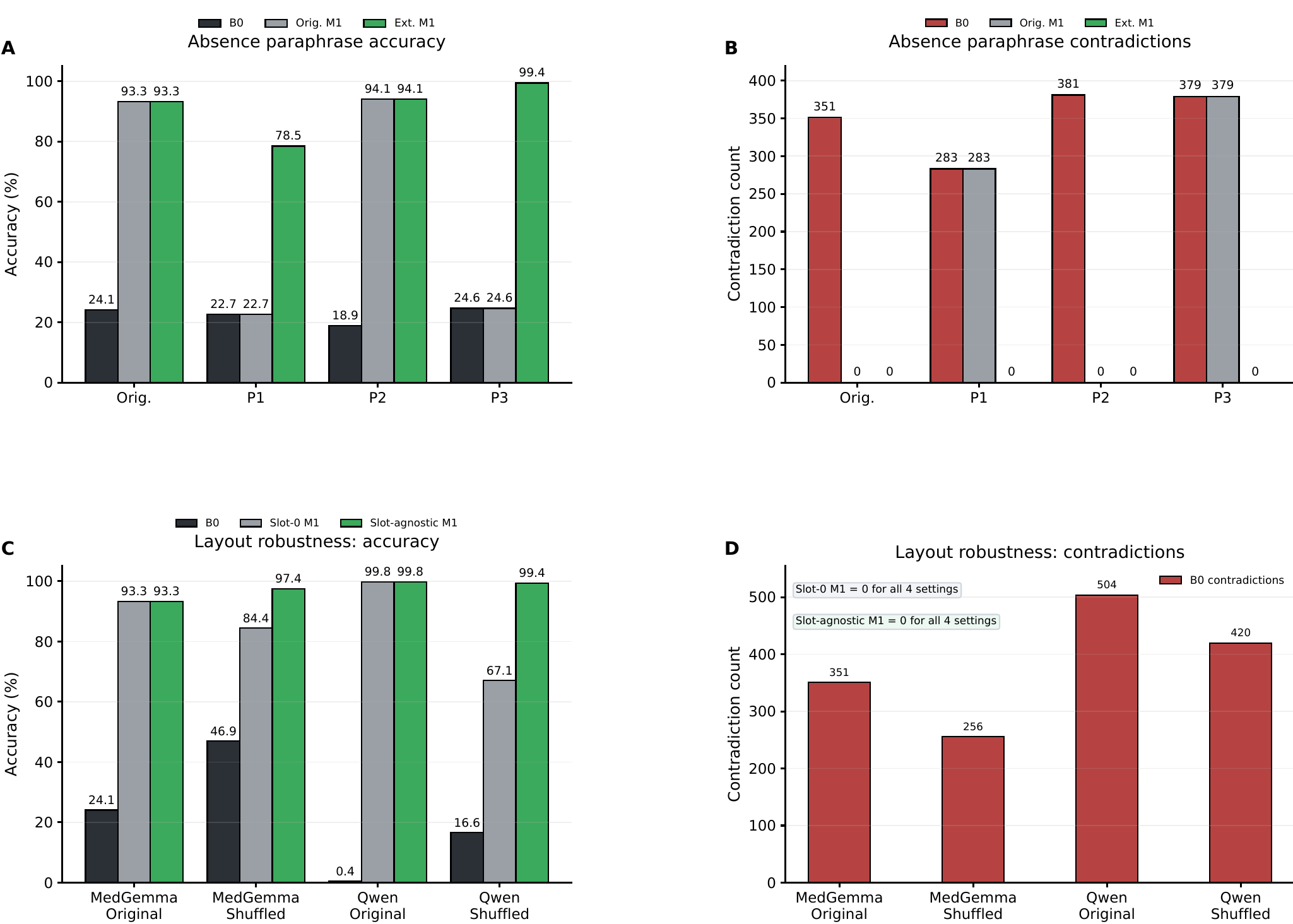}
  \caption{Robustness atlas for paraphrase and layout perturbations. The top row reports accuracy and contradiction behavior under paraphrase variants with the original and extended triggers. The bottom row reports accuracy and contradiction behavior under original and shuffled option layouts with slot-0 and slot-agnostic M1.}
  \label{fig:app-robustness-atlas}
\end{figure*}

\subsection{Finding-Level and Error-Audit Breakdown}

The direct presence probe and the absence-side protocols both show meaningful finding-level heterogeneity. Table~\ref{tab:presence-finding-breakdown} gives the five open-weight backbones with complete per-finding direct presence-side ``No $X$'' reversal rates; aggregate large-backbone results for MedGemma-27B and GPT-4o are reported in Appendix~\ref{app:large_models}. Within the five-backbone breakdown, pneumothorax has the highest average reversal rate and consolidation shows one of the sharpest cross-model spreads. Table~\ref{tab:finding-breakdown} then reports the complementary absence-side contradiction summary aggregated over the eight model and prompt variants in Figure~\ref{fig:external-results}. There, edema has the highest average contradiction rate (70.56\%), followed by consolidation (64.75\%) and pleural effusion (64.52\%). The benchmark therefore captures clinically meaningful heterogeneity rather than a single universal ordering.

\begin{table}[h]
\caption{Direct CheXpert presence-probe negated-option reversal rates on B0 by target finding and model. Rates are percentages within each finding-specific subset, and 0 values indicate that the model did not emit the negated option for that finding on this protocol.}
\label{tab:presence-finding-breakdown}
\centering
\scriptsize
\resizebox{\linewidth}{!}{%
\begin{tabular}{lrrrrrr}
\toprule
Model & Atelectasis & Cardiomegaly & Consolidation & Edema & Pleural effusion & Pneumothorax \\
\midrule
CheXagent & 1.61 & 0.00 & 3.23 & 0.00 & 8.93 & 0.00 \\
LLaVA-1.5 & 6.45 & 27.45 & 0.00 & 24.24 & 1.79 & 0.00 \\
LLaVA-Med & 0.00 & 0.00 & 0.00 & 0.00 & 0.00 & 0.00 \\
MedGemma & 11.29 & 80.39 & 100.00 & 78.79 & 82.14 & 100.00 \\
Qwen2.5-VL & 58.06 & 64.71 & 41.94 & 72.73 & 76.79 & 100.00 \\
\bottomrule
\end{tabular}}
\end{table}

\begin{table}[h]
\caption{CheXpert finding-level contradiction rates on B0, aggregated over the eight model/prompt variants in Figure~\ref{fig:external-results}. The ranking is not uniform across models, but edema is the highest-rate category on average.}
\label{tab:finding-breakdown}
\centering
\small
\begin{tabular}{lrrrr}
\toprule
Finding & $n$ & Mean rate & Min & Max \\
\midrule
Pneumothorax & 126 & 56.65 & 0.00 & 99.21 \\
Consolidation & 100 & 64.75 & 10.00 & 100.00 \\
Edema & 90 & 70.56 & 0.00 & 100.00 \\
Pleural effusion & 68 & 64.52 & 5.88 & 100.00 \\
Cardiomegaly & 66 & 57.95 & 3.03 & 100.00 \\
Atelectasis & 57 & 59.87 & 26.32 & 100.00 \\
\bottomrule
\end{tabular}
\end{table}

Figure~\ref{fig:app-finding-error-atlas} presents the full structured error analysis across both direct presence probes and absence-side protocols. The model-by-finding heatmap~(Panel~A) shows that contradiction burden clusters along both model and finding dimensions. Panel~B provides a count-based summary of dangerous negated-option errors by finding. Panel~C reports the composition of a 100-case manual audit; negation-related error categories dominate, but the slice is not monolithic. Panel~D maps audit categories onto deterministic fixability, explaining why a lightweight verifier can resolve many but not all observed failures.

\begin{figure*}[htbp]
  \centering
  \includegraphics[width=\linewidth]{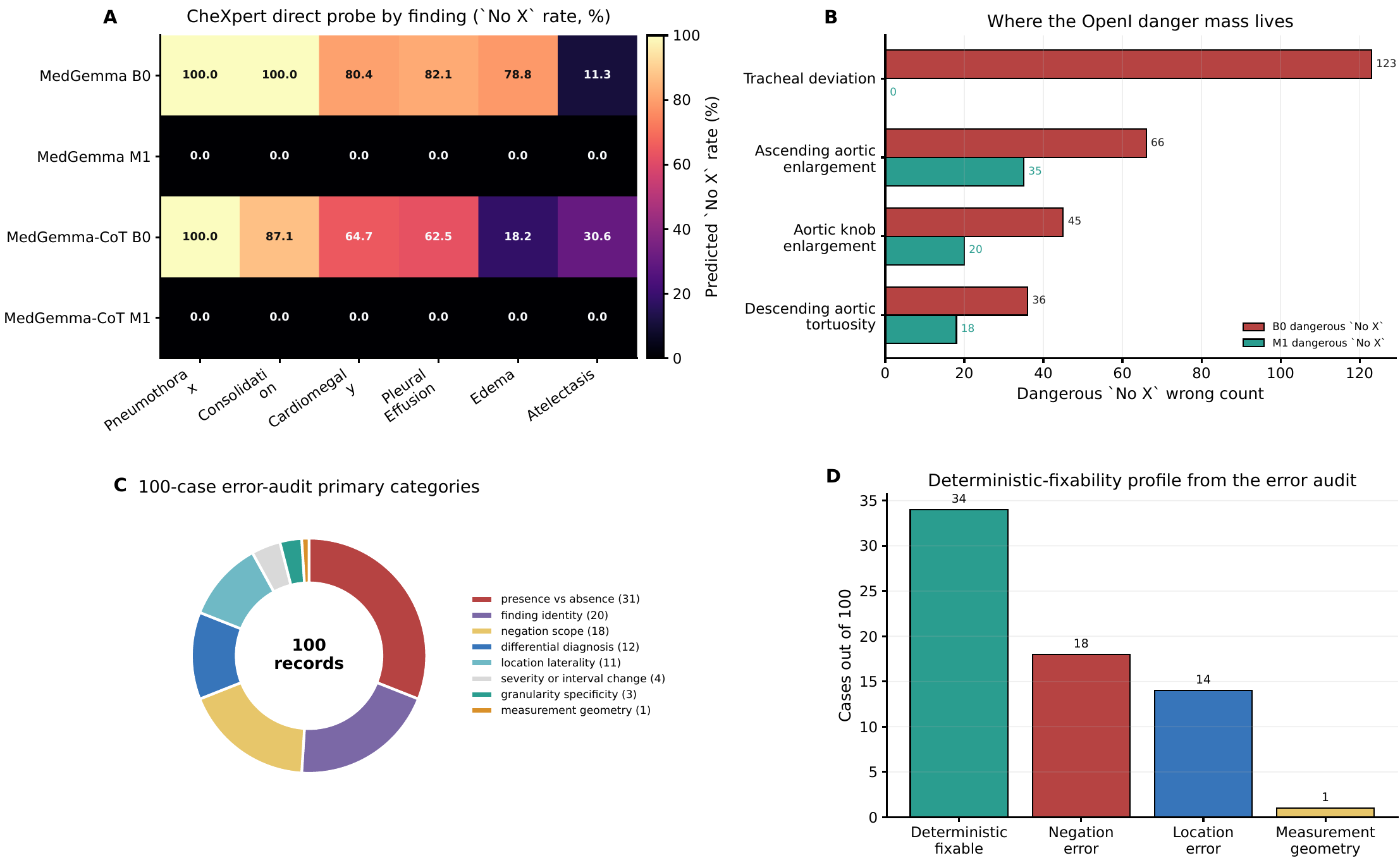}
  \caption{Finding and error atlas. The panels show model-by-finding contradiction heatmaps, contradiction statistics by finding (including both rates and counts), the 100-case error-audit composition, and the deterministic-fixability profile induced by the audit labels.}
  \label{fig:app-finding-error-atlas}
\end{figure*}

\subsection{Public Case Gallery}

Figure~\ref{fig:app-casebook-gallery} provides a cue-first public casebook of representative OpenI cases selected to illustrate the benchmark's repair behavior. The figure groups cases into presence-side repairs and absence-side repairs. In these cards, report-side cue evidence is foregrounded, the polarity mapping is stated explicitly, and the original benchmark item is shown only for reference.

\begin{figure*}[htbp]
  \centering
  \includegraphics[width=\linewidth]{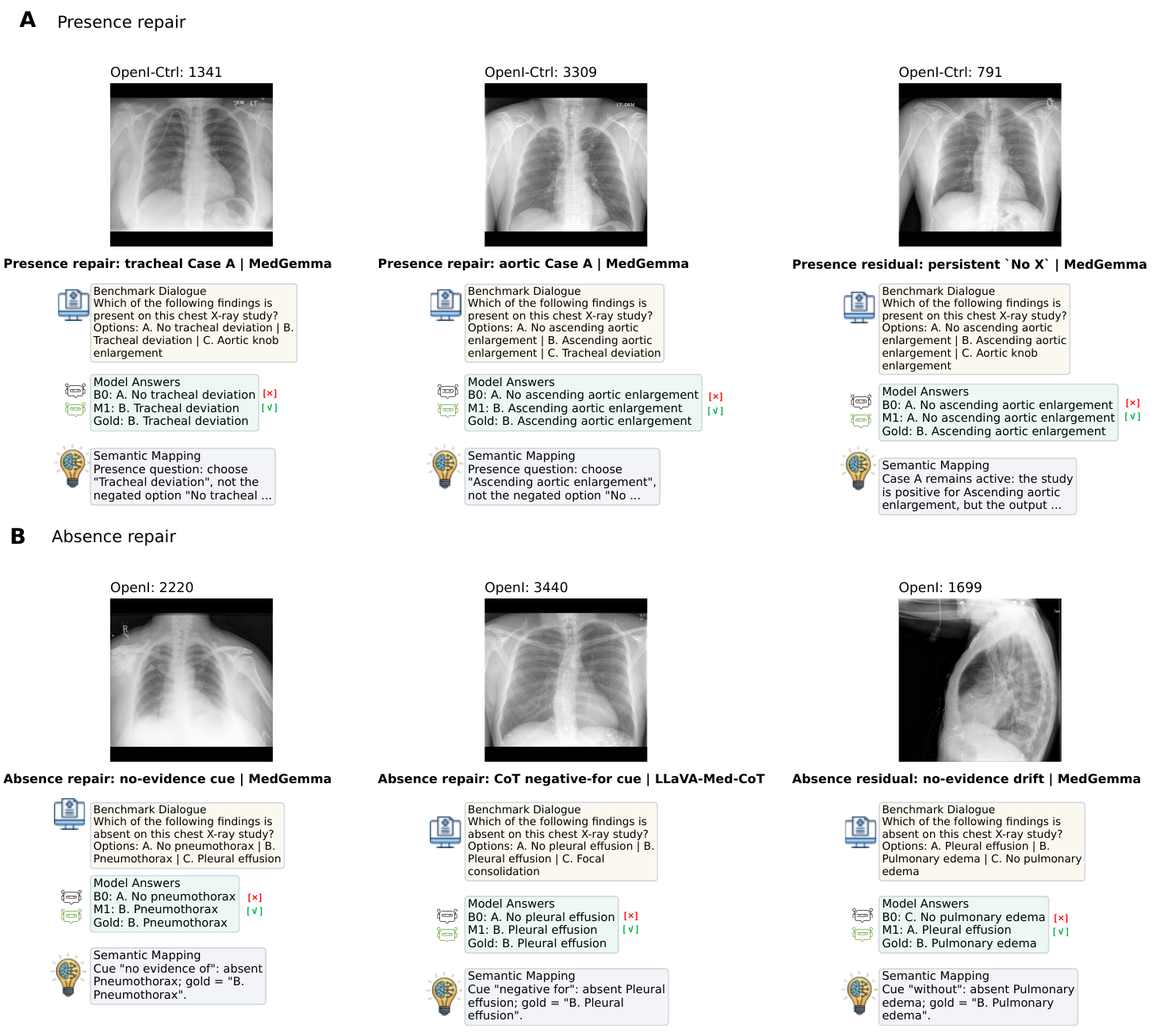}
  \caption{Public report-negation casebook with deterministic OpenI case selection. Each card foregrounds report-side cue evidence, a compact polarity-mapping summary, and raw B0/M1/Gold options, while the original benchmark item is shown only as a reference.}
  \label{fig:app-casebook-gallery}
\end{figure*}

\section{Large-Backbone Evaluation: MedGemma-27B and GPT-4o}
\label{app:large_models}

We provide extended analysis of MedGemma-27B-it and GPT-4o across all three external evaluation settings, including the direct presence probe, the matched positive-only control, and the absence protocols. All evaluations use identical prompting and
\method{} post-hoc verification where applicable. MedGemma-27B-it is evaluated
locally with \texttt{device\_map="auto"} across four GPUs and
\texttt{bfloat16} precision; \texttt{max\_new\_tokens} is set to
64 to accommodate the model's longer generation style (the
default 16 used for 4B caused severe answer truncation, yielding
spuriously low accuracy). GPT-4o is evaluated via API with
temperature 0 and a single-image-per-study constraint matching the
other backbones. For direct presence we report B0 and M1; for the
absence protocols we additionally report B0-Instruct (polarity
instruction prepended). Parse-failure rates are
additionally reported for MedGemma-27B to distinguish genuine
accuracy from truncation artifacts.

\subsection{Direct Presence Extension and Matched Positive-Only Control}

Extending the direct CheXpert presence probe to the larger models
shows that stronger models are less failure-prone on this construction,
but not exempt from presence-side semantic reversal. On the full
235-record probe, MedGemma-27B reaches 51.91\% B0 accuracy with
105 dangerous negated-option reversals; \method{} repairs all of them
and raises accuracy to 96.60\%. GPT-4o starts higher at 62.55\% B0 but
still emits 25 reversals, improving to 73.19\% after deterministic
repair. The failure therefore persists beyond the smaller open-weight
backbones emphasized in the main text, although its magnitude becomes
substantially more backbone-dependent.

Table~\ref{tab:presence_positive_control} reports the matched
210-record positive-only control. Removing the negated option increases
accuracy for MedGemma-4B (+26.67~pp) and MedGemma-27B (+10.00~pp),
consistent with direct attraction to the negated surface form. By
contrast, Qwen2.5-VL also improves on the control (+16.19~pp), while
GPT-4o performs worse (-16.67~pp). The control therefore sharpens causal
interpretation in a more selective way: negated-option attraction is a
major driver for MedGemma-4B, Qwen2.5-VL, and MedGemma-27B, whereas
GPT-4o retains broader present-finding selection errors and may use the
negated option as an exclusion anchor during answer elimination.

\begin{table}[h]
\caption{Matched positive-only control for the direct CheXpert presence setting. The comparison uses the 210-record subset for which a paired control item exists, so the direct-presence accuracy in this table differs slightly from the full 235-record direct presence results in Figure~\ref{fig:external-results}.}
\label{tab:presence_positive_control}
\centering
\small
\begin{tabular}{lrrrrrr}
\toprule
Model & $n$ & Direct-presence acc. & Control acc. & Direct ``No $X$'' & Direct ``No $X$'' rate \\
\midrule
MedGemma-4B & 210 & 30.00 & 56.67 & +26.67 & 139 & 66.19 \\
Qwen2.5-VL & 210 & 30.00 & 46.19 & +16.19 & 136 & 64.76 \\
MedGemma-27B & 210 & 50.95 & 60.95 & +10.00 & 95 & 45.24 \\
GPT-4o & 210 & 62.38 & 45.71 & -16.67 & 20 & 9.52 \\
\bottomrule
\end{tabular}
\end{table}

\subsection{Scale Does Not Uniformly Reduce Negation Attraction}

Table~\ref{tab:large_model_results} summarizes the absence-side results
across all four models and two protocols. Several findings stand out.

\paragraph{MedGemma 4B~$\to$~27B.}
On CheXpert, MedGemma-27B B0 accuracy (31.16\%) is comparable to
MedGemma-4B (24.06\%), with a similar contradiction burden (335
vs.\ 351). On OpenI, however, 27B B0 accuracy (15.32\%) falls well
below 4B (70.47\%), indicating that scale does not consistently
reduce negated-option attraction and can worsen it on some
distributions. The instruction variant for 27B substantially
reduces contradictions (PARTIAL classification on both protocols),
unlike 4B where instruction amplifies contradictions (BACKFIRE on
both). \method{} eliminates all measured contradictions for 27B
on CheXpert (M1 97.24\%) but provides more modest recovery on
OpenI (M1 52.37\%), likely because the higher parse-failure rate
(6.69\%) limits the number of correctable negation predictions.

\paragraph{GPT-4o.}
GPT-4o shows the most severe B0 negation attraction of any tested
backbone: 470 contradictions on CheXpert (B0 accuracy 2.17\%) and
284 on OpenI (16.16\%), indicating that high general capability
does not prevent polarity confusion in this answer-space regime.
Instruction prompting is substantially more effective for GPT-4o
than for MedGemma: accuracy rises to 57.20\% on CheXpert and
67.97\% on OpenI, with classification PARTIAL on both. Combined
Instruct+M1 reaches 83.83\% on CheXpert and 68.25\% on OpenI.
\method{} applied to plain B0 eliminates all contradictions on
both protocols (M1 94.87\% and 63.79\%, respectively).

\paragraph{M1 universality.}
Across all four B0 runs in this extended evaluation (two models
$\times$ two protocols $\times$ B0), \method{} eliminates every
measured contradiction and yields a positive accuracy delta in
every case, consistent with the main-paper results for smaller
backbones. This confirms that the deterministic four-condition
trigger generalizes across model families and scales without
modification.

\begin{table*}[htbp]
\caption{Large-backbone results on CheXpert and OpenI report-style
absence protocols. Parse-failure rates are reported for MedGemma-27B
to distinguish truncation artifacts from genuine errors.
Instruct classification: BACKFIRE = instruction increases
contradictions; PARTIAL = instruction reduces but does not
eliminate contradictions; FULL\_SOLVE = instruction eliminates
contradictions. All M1 results apply \method{} to B0 outputs.}
\label{tab:large_model_results}
\centering
\scriptsize
\resizebox{\textwidth}{!}{%
\begin{tabular}{llrrrrrrrr}
\toprule
Model & Protocol & $n$ &
  B0 Acc. & B0 Contra. &
  Instruct Acc. & Instruct Contra. & Instruct Class. &
  M1 Acc. & M1 Contra. \\
\midrule
MedGemma-4B & CheXpert & 507 &
  24.06 & 351 & 12.23 & 406 & BACKFIRE & 93.29 & 0 \\
MedGemma-4B & OpenI & 359 &
  70.47 & 79  & 49.03 & 162 & BACKFIRE & 86.91 & 0 \\
\midrule
MedGemma-27B$^\dagger$ & CheXpert & 507 &
  31.16 & 335 & 77.32 & 18 & PARTIAL  & 97.24 & 0 \\
MedGemma-27B$^\dagger$ & OpenI & 359 &
  15.32 & 199 & 43.73 & 38 & PARTIAL  & 52.37 & 0 \\
\midrule
GPT-4o & CheXpert & 507 &
  2.17 & 470 & 57.20 & 135 & PARTIAL & 94.87 & 0 \\
GPT-4o & OpenI & 359 &
  16.16 & 284 & 67.97 & 1   & PARTIAL & 63.79 & 0 \\
\bottomrule
\end{tabular}}
\vspace{0.3em}

{\footnotesize $^\dagger$ MedGemma-27B parse-failure rates:
CheXpert B0 0.79\%, CheXpert Instruct 12.03\%,
OpenI B0 6.69\%, OpenI Instruct 12.26\%.
All MedGemma-27B contradiction counts are reported because the
corresponding parse-failure rates remain below 15\%.}
\end{table*}

\subsection{Implications}

These results extend three conclusions from the main paper.
First, negated-option attraction is not resolved by scaling:
MedGemma-27B and GPT-4o both show substantial B0 contradiction
rates, and GPT-4o's B0 contradiction rate is the highest of any
backbone evaluated. Second, instruction sensitivity is
backbone-dependent: GPT-4o responds strongly to polarity
instruction, while MedGemma-4B is harmed by it (BACKFIRE), and
MedGemma-27B shows an intermediate response (PARTIAL). Third,
\method{} remains effective across all tested scales and
architectures, eliminating contradictions without requiring
model-specific adaptation. The parse-failure issue observed for
MedGemma-27B highlights that evaluation infrastructure designed
for short-answer models requires adaptation when applied to
larger models with longer generation style.

\section{Few-Shot In-Context Learning Baseline}
\label{app:fewshot}

We evaluate three-shot in-context learning for MedGemma-4B and Qwen2.5-VL on the 235-record direct CheXpert presence probe and the 507-record CheXpert report-style absence protocol. The intended construction was a held-out CheXpert validation pool restricted to studies outside the evaluation protocols. In practice, once the 82 direct-presence and 132 absence-side evaluation studies are excluded, the CheXpert validation split contains no constructible held-out examples under the benchmark rules. The final few-shot pools are therefore built from non-overlapping OpenI sources: 20 CheXStruct presence examples for the presence probe and 20 OpenI report-style absence examples for the absence probe, with three examples sampled deterministically per test call (seed 42) and at least two findings when possible.

This setup is intentionally conservative for the main claim. The exemplars are polarity-correct and non-overlapping, but they are cross-dataset rather than within-CheXpert, so any gain cannot be attributed to same-distribution memorization. The resulting table therefore tests whether lightweight prompt-only exemplar conditioning can suppress the answer-space polarity error without retraining, even when the exemplars come from a different source distribution.

\begin{table*}[htbp]
\caption{Few-shot in-context learning comparison on the two CheXpert protocols for MedGemma-4B and Qwen2.5-VL. `neg\_opt\_count` denotes direct-presence ``No $X$'' predictions or absence-side contradictions, and `changed`/`worsened` are measured against the paired FS+M1 application. Baseline and CoT reference rows remain in Figure~\ref{fig:external-results}; this appendix table isolates the new few-shot intervention.}
\label{tab:fewshot_comparison}
\centering
\scriptsize
\resizebox{\textwidth}{!}{%
\begin{tabular}{lllrrrr}
\toprule
Model & Protocol & Variant & Accuracy & neg\_opt\_count & changed & worsened \\
\midrule
MedGemma-4B-it & CheXpert presence & FS-B0 & 29.36 & 160 & 160 & 0 \\
MedGemma-4B-it & CheXpert presence & FS+M1 & 97.45 & 0 & 160 & 0 \\
MedGemma-4B-it & CheXpert absence & FS-B0 & 2.17 & 486 & 486 & 0 \\
MedGemma-4B-it & CheXpert absence & FS+M1 & 98.03 & 0 & 486 & 0 \\
Qwen2.5-VL-7B-Instruct & CheXpert presence & FS-B0 & 29.79 & 150 & 150 & 0 \\
Qwen2.5-VL-7B-Instruct & CheXpert presence & FS+M1 & 93.62 & 0 & 150 & 0 \\
Qwen2.5-VL-7B-Instruct & CheXpert absence & FS-B0 & 28.40 & 354 & 354 & 0 \\
Qwen2.5-VL-7B-Instruct & CheXpert absence & FS+M1 & 98.22 & 0 & 354 & 0 \\
\bottomrule
\end{tabular}}
\end{table*}

Three conclusions follow. First, few-shot ICL is not a stable direct-presence remedy: it slightly worsens raw accuracy for both MedGemma and Qwen and leaves essentially the full dangerous reversal burden intact. Second, on the absence-side it is highly backbone-dependent, catastrophically harming MedGemma while helping Qwen, so prompt-only exemplars do not provide a predictable intervention. Third, FS+M1 still eliminates the measured polarity-confused subset in all four settings, indicating that the deterministic verifier remains the only consistently effective intervention among the prompt-level baselines tested here.

\section{In-Format LoRA Fine-tuning Reference}
\label{app:lora}

As a training-time reference, we fine-tune MedGemma-4B-it and Qwen2.5-VL-7B-Instruct with LoRA (rank 16) on a balanced 4{,}000-example CheXpert training-split set containing 2{,}000 presence questions and 2{,}000 absence questions. The training set uses the same answer-space format and question templates as the evaluation protocols, with zero study overlap from all validation evaluation studies. This experiment is therefore intended to test whether explicit in-format supervision can suppress the measured polarity failure, not whether fine-tuning yields transferable visual negation understanding. LoRA+M1 denotes M1 applied post hoc to LoRA-FT predictions.

\begin{table}[h]
\caption{Comparison of prompt-level and training-level baselines on the direct CheXpert presence probe ($n{=}235$) and report-style absence protocol ($n{=}507$). \textit{neg\_opt} denotes dangerous ``No $X$'' reversals (presence) or contradictions (absence). LoRA-FT and LoRA+M1 are evaluated on the validation split; training used only CheXpert training-split studies.}
\label{tab:lora_comparison}
\centering
\scriptsize
\begin{tabular}{llrr}
\toprule
Model & Variant & Presence acc. / neg\_opt & Absence acc. / neg\_opt \\
\midrule
MedGemma-4B & B0        & 31.49 / 153 & 23.08 / 356 \\
             & CoT-B0    & 44.26 / 110 & 33.14 / 285 \\
             & FS-B0     & 29.36 / 160 & 2.17 / 486  \\
             & M1        & 96.60 / 0   & 93.29 / 0   \\
             & LoRA-FT   & 100.00 / 0  & 100.00 / 0  \\
             & LoRA+M1   & 100.00 / 0  & 100.00 / 0  \\
\midrule
Qwen2.5-VL  & B0        & 30.21 / 153 & 0.20 / 505  \\
             & CoT-B0    & 56.60 / 58  & 27.22 / 368 \\
             & FS-B0     & 29.79 / 150 & 28.40 / 354 \\
             & M1        & 95.32 / 0   & 99.80 / 0   \\
             & LoRA-FT   & 100.00 / 0  & 100.00 / 0  \\
             & LoRA+M1   & 100.00 / 0  & 100.00 / 0  \\
\bottomrule
\end{tabular}
\end{table}

LoRA-FT removes the measured polarity-confused outputs in all four model-protocol settings, and LoRA+M1 yields the same zero-error measured subset. This result should be interpreted narrowly. It shows that explicit in-format supervised adaptation can suppress the benchmarked answer-space failure, but it does not by itself establish transferable visual polarity understanding.

\paragraph{Format-specific adaptation.}
The final training loss is near zero for both models (MedGemma $1.2 \times 10^{-5}$; Qwen $3 \times 10^{-6}$), indicating that the models fit the 4{,}000-example training set almost completely. Because training and evaluation share the same question templates and answer-space layout, a model can solve the measured task by learning a surface policy such as avoiding options beginning with ``No'' for presence questions and choosing the concept name for absence questions. The zero study overlap rules out direct image-level leakage, but it does not rule out format-specific adaptation. Therefore, this experiment is best viewed as an in-format supervised reference rather than evidence that the model has acquired robust clinical negation reasoning.

\paragraph{Scope relative to \method{}.}
This comparison also clarifies the role of \method{}. Within the explicit-negation regime measured by CXR-ContraBench, a deterministic post-hoc verifier with no training data and no extra model calls reaches within $3$--$7$~pp of the in-format LoRA reference on the validation protocols, while remaining transparent about exactly which predictions it changes. LoRA-FT is stronger under matched supervision, but it requires labeled in-domain examples and may learn the benchmark format. \method{} is therefore not positioned as a replacement for supervised adaptation; it is a training-free audit and repair mechanism for the measured polarity-confused subset.

\paragraph{Boundary under broader negation.}
The paraphrase stress test further sharpens this scope. The measured failure persists beyond the literal ``No $X$'' template, while the original literal-surface verifier loses coverage on many non-canonical forms. A paraphrase-aware deterministic extension restores the measurable subset without observed regressions, but implicit or compositional negation remains outside the scope of the current verifier. Whether broader fine-tuning would learn transferable semantic polarity rather than a wider family of surface rules remains an open question.

\section{Negation Paraphrase Stress Test}
\label{app:paraphrase}

We next test the strongest surface-form objection to CXR-ContraBench: that the measured failure may be tied only to the literal option form ``No $X$'' and may therefore overestimate a template artifact. To isolate that question, we keep the underlying images, question polarity, distractor structure, and scoring protocol fixed, but rewrite the negated option using four non-canonical variants: ``Absence of $X$,'' ``$X$ is not present,'' ``No evidence of $X$,'' and ``Clear of $X$.'' We evaluate the direct CheXpert presence probe and the CheXpert report-style absence protocol under three systems: baseline B0, the original literal-surface verifier (Orig-M1), and a paraphrase-aware deterministic extension (Ext-M1). Ext-M1 preserves the same intervention rule as the canonical verifier---one matched negated option, baseline predicts it, polarity is safely identified, and a unique positive counterpart exists---but expands the hand-written negation matcher to the audited paraphrase set. No additional model call, learned classifier, or search procedure is introduced.

Table~\ref{tab:paraphrase_presence} shows the presence-side results. The key pattern is that B0 polarity confusion persists under every non-canonical rewrite for both MedGemma-4B-it and Qwen2.5-VL-7B-Instruct. Across the eight non-canonical presence runs, the original literal-surface verifier loses all coverage and therefore leaves the B0 reversals unchanged, while Ext-M1 restores the measurable subset in every run and reduces dangerous reversals to zero throughout. The strongest B0 failures occur for MedGemma under ``No evidence of $X$'' (215 reversals) and ``Clear of $X$'' (214 reversals), and for Qwen under ``$X$ is not present'' (188 reversals). Canonical ``No $X$'' results match the original main-text numbers exactly.

\begin{table*}[t]
\caption{Direct-presence negation-paraphrase stress test on the 235-record CheXpert presence probe. Counts are presence-side semantic reversals. Orig-M1 is the original literal-surface verifier; Ext-M1 is the paraphrase-aware deterministic extension.}
\label{tab:paraphrase_presence}
\centering
\scriptsize
\resizebox{\textwidth}{!}{%
\begin{tabular}{llrrrrrrrrrr}
\toprule
Model & Variant & $n$ & B0 acc & B0 rev. & Orig-M1 acc & Orig-M1 rev. & Ext-M1 acc & Ext-M1 rev. & Orig-M1 worsened & Ext-M1 worsened \\
\midrule
MedGemma-4B-it & canonical\_no   & 235 & 31.49 & 153 & 96.60 &   0 & 96.60 &   0 & 0 & 0 \\
MedGemma-4B-it & absence\_of     & 235 & 21.70 & 178 & 21.70 & 178 & 97.45 &   0 & 0 & 0 \\
MedGemma-4B-it & not\_present    & 235 & 39.15 & 124 & 39.15 & 124 & 91.91 &   0 & 0 & 0 \\
MedGemma-4B-it & no\_evidence\_of& 235 &  7.66 & 215 &  7.66 & 215 & 99.15 &   0 & 0 & 0 \\
MedGemma-4B-it & clear\_of       & 235 &  8.09 & 214 &  8.09 & 214 & 99.15 &   0 & 0 & 0 \\
Qwen2.5-VL-7B-Instruct & canonical\_no    & 235 & 30.21 & 153 & 95.32 &   0 & 95.32 &   0 & 0 & 0 \\
Qwen2.5-VL-7B-Instruct & absence\_of      & 235 & 33.19 & 140 & 33.19 & 140 & 92.77 &   0 & 0 & 0 \\
Qwen2.5-VL-7B-Instruct & not\_present     & 235 & 19.57 & 188 & 19.57 & 188 & 99.57 &   0 & 0 & 0 \\
Qwen2.5-VL-7B-Instruct & no\_evidence\_of & 235 & 25.96 & 168 & 25.96 & 168 & 97.45 &   0 & 0 & 0 \\
Qwen2.5-VL-7B-Instruct & clear\_of        & 235 & 40.00 & 132 & 40.00 & 132 & 96.17 &   0 & 0 & 0 \\
\bottomrule
\end{tabular}}
\end{table*}

Table~\ref{tab:paraphrase_absence} shows the complementary absence-side protocol. Here the original verifier partially generalizes: it already covers ``No evidence of $X$'' and ``Clear of $X$'' in this construction, but fails on ``Absence of $X$'' and ``$X$ is not present.'' Across the eight non-canonical absence runs, B0 still shows measurable contradiction burden in all of them, the original literal-surface verifier retains nonzero contradictions in four runs, and Ext-M1 removes the measurable contradiction subset in all eight. Taken together with the presence-side results, this yields a clean rebuttal pattern: the benchmarked failure is broader than a single literal template, the original verifier is narrower than the underlying phenomenon, and a still-deterministic paraphrase-aware matcher restores coverage without regressions.

\begin{table*}[t]
\caption{Report-style absence negation-paraphrase stress test on the 507-record CheXpert absence protocol. Counts are absence-side negation contradictions. Orig-M1 is the original literal-surface verifier; Ext-M1 is the paraphrase-aware deterministic extension.}
\label{tab:paraphrase_absence}
\centering
\scriptsize
\resizebox{\textwidth}{!}{%
\begin{tabular}{llrrrrrrrrrr}
\toprule
Model & Variant & $n$ & B0 acc & B0 contr. & Orig-M1 acc & Orig-M1 contr. & Ext-M1 acc & Ext-M1 contr. & Orig-M1 worsened & Ext-M1 worsened \\
\midrule
MedGemma-4B-it & canonical\_no   & 507 & 23.08 & 356 & 93.29 &   0 & 93.29 &   0 & 0 & 0 \\
MedGemma-4B-it & absence\_of     & 507 & 18.34 & 379 & 18.34 & 379 & 93.10 &   0 & 0 & 0 \\
MedGemma-4B-it & not\_present    & 507 & 22.68 & 349 & 22.68 & 349 & 91.52 &   0 & 0 & 0 \\
MedGemma-4B-it & no\_evidence\_of& 507 & 28.21 & 326 & 92.50 &   0 & 92.50 &   0 & 0 & 0 \\
MedGemma-4B-it & clear\_of       & 507 & 41.42 & 221 & 85.01 &   0 & 85.01 &   0 & 0 & 0 \\
Qwen2.5-VL-7B-Instruct & canonical\_no    & 507 &  0.20 & 505 & 99.80 &   0 & 99.80 &   0 & 0 & 0 \\
Qwen2.5-VL-7B-Instruct & absence\_of      & 507 &  0.00 & 507 &  0.00 & 507 & 100.00 &   0 & 0 & 0 \\
Qwen2.5-VL-7B-Instruct & not\_present     & 507 &  0.00 & 505 &  0.00 & 505 & 99.61 &   0 & 0 & 0 \\
Qwen2.5-VL-7B-Instruct & no\_evidence\_of & 507 &  0.20 & 506 & 100.00 &   0 & 100.00 &   0 & 0 & 0 \\
Qwen2.5-VL-7B-Instruct & clear\_of        & 507 & 15.19 & 377 & 89.55 &   0 & 89.55 &   0 & 0 & 0 \\
\bottomrule
\end{tabular}}
\end{table*}

Across the 16 non-canonical runs in Tables~\ref{tab:paraphrase_presence} and~\ref{tab:paraphrase_absence}, B0 retains measurable polarity-confused predictions in all 16, showing that the failure is not confined to the literal ``No $X$'' string. The original literal-surface verifier loses coverage on 12 of the 16 runs, exactly where the option form falls outside its audited trigger set. By contrast, Ext-M1 restores coverage on all 16 and introduces no worsened cases; the canonical ``No $X$'' regression checks also match the original reported numbers exactly. The strongest interpretation is therefore narrow but robust: CXR-ContraBench is measuring a broader polarity-sensitivity failure, while the original main-text verifier is a deliberately conservative literal-surface instantiation of a more general deterministic matching principle.

In summary, this section shows that the measured failure is not restricted to the literal ``No $X$'' template. The failure persists under broader non-canonical negation wording, the original literal-surface verifier loses coverage on many such variants, and the extended deterministic matcher restores the measurable subset without observed regressions on this experiment.

\section{Responsible Release, Licenses, and Broader Impact}
\label{app:licenses}

\paragraph{Responsible release.}
We provide a submission-time anonymous release of the protocol JSON files, benchmark builders, evaluation scripts, deterministic verification logic, and aggregate metric utilities at \url{https://anonymous.4open.science/r/cxr-contrabench}. The release does not bundle raw images; all raw data remain subject to the licenses and access controls of the original datasets. This design minimizes redistribution risk while preserving reproducibility of the benchmark logic and measured claims.

\paragraph{Existing assets and licenses.}
ReXVQA is distributed under a non-commercial data access and use agreement~\citep{rexvqa}. VQA-RAD is released under CC0~1.0~\citep{vqarad}. CheXStruct, used through the CXReasonBench release, is distributed under CC~BY~4.0~\citep{cxreasonbench}. CheXpert and OpenI retain their original dataset terms~\citep{chexpert,openi}. MedGemma is provided under the Health AI Developer Foundations license~\citep{medgemma}. Our derived benchmark assets are designed to reference these sources rather than supersede them.

\paragraph{Broader impact.}
The primary positive contribution is making negated-option attraction measurable in a clinically sensitive setting where the failure has direct diagnostic consequences. A model that selects ``No pleural effusion'' when asked which finding is present emits the opposite clinical statement. This is a silent error that aggregate accuracy can hide. The main risk is misinterpretation: benchmark gains from \method{} should not be read as proof of clinical readiness. CXR-ContraBench measures one well-defined failure mode with scale-confirmed prevalence; end-to-end clinical validation requires additional prospective testing beyond the benchmark's scope, and \method{} should be understood as a diagnostic mitigation rather than a deployment guarantee.
\end{document}